\documentclass[letterpaper]{article} 
\usepackage[preprint]{aaai2027}  

\usepackage[hyphens]{url}  
\usepackage{graphicx} 
\usepackage{natbib}  
\usepackage{caption} 
\usepackage{xcolor}
\usepackage{tikz}
\usetikzlibrary{positioning, arrows.meta, calc, fit, backgrounds, shapes.geometric, decorations.pathreplacing}
\usepackage{multirow}

\usepackage{algorithm}
\usepackage{algorithmic}

\usepackage{amsmath}
\usepackage{amssymb}
\usepackage{mathtools}
\usepackage{amsthm}

\usepackage[capitalize,noabbrev]{cleveref}

\theoremstyle{plain}
\newtheorem{theorem}{Theorem}
\newtheorem{corollary}{Corollary}
\newtheorem{proposition}[theorem]{Proposition}
\newtheorem{lemma}[theorem]{Lemma}

\usepackage{framed}
\definecolor{shadecolor}{gray}{0.93}
\let\origtheorem\theorem     \let\origendtheorem\endtheorem
\renewenvironment{theorem}{\begin{snugshade}\origtheorem}{\origendtheorem\end{snugshade}}
\let\origproposition\proposition \let\origendproposition\endproposition
\renewenvironment{proposition}{\begin{snugshade}\origproposition}{\origendproposition\end{snugshade}}
\let\origlemma\lemma         \let\origendlemma\endlemma
\renewenvironment{lemma}{\begin{snugshade}\origlemma}{\origendlemma\end{snugshade}}
\let\origcorollary\corollary \let\origendcorollary\endcorollary
\renewenvironment{corollary}{\begin{snugshade}\origcorollary}{\origendcorollary\end{snugshade}}

\newtheorem{reprop}{Proposition}
\newtheorem{recor}{Corollary}
\let\origreprop\reprop \let\origendreprop\endreprop
\renewenvironment{reprop}{\begin{snugshade}\origreprop}{\origendreprop\end{snugshade}}
\let\origrecor\recor \let\origendrecor\endrecor
\renewenvironment{recor}{\begin{snugshade}\origrecor}{\origendrecor\end{snugshade}}

\DeclareMathOperator*{\argmax}{arg\,max}

\newcommand{\std}[1]{{\scriptsize$\pm$#1}}
\newcommand{\chimera}{ChiMERa-Bench}
\newcommand{\af}{\emph{AgForce}}

\usepackage{newfloat}
\usepackage{listings}
\DeclareCaptionStyle{ruled}{labelfont=normalfont,labelsep=colon,strut=off} 
\floatstyle{ruled}
\newfloat{listing}{tb}{lst}{}
\floatname{listing}{Listing}

\usepackage{booktabs}

\title{AgForce Enables Antigen-conditioned Generative Antibody Design}

\author{
    Mansoor Ahmed\textsuperscript{\rm 1,2}\corresponding,
    Murray Patterson\textsuperscript{\rm 1}\corresponding
}
\affiliations{
    \textsuperscript{\rm 1}Georgia State University, Atlanta, USA\\
    \textsuperscript{\rm 2}Georgia Institute of Technology, Atlanta, USA\\
    mahmed76@student.gsu.edu, mpatterson30@gsu.edu
}

\begin{document}
\maketitle

\begin{abstract}
Antibody design methods condition on antigen structure to generate complementarity-determining regions (CDR), yet a systematic evaluation of baseline methods reveals that they largely ignore the antigen input. We identify three failure modes that explain this behavior. \textit{Antigen blindness} arises because models derive predictions from antibody framework context rather than antigen information, producing nearly identical CDRs regardless of the target. \textit{Vocabulary collapse} reduces predicted amino acids to three to five per position, far below the ground truth distribution in native sequences. The \textit{argmax bottleneck} on the cross-entropy loss via greedy decoding acts as a null predictor that sees only position and loop length and recovers almost all the residues irrespective of the model architecture and the target antigen. We propose a novel encoder-decoder architecture called \af{} that uses a graph neural network (GNN) as the encoder and specialized decoders for sequence-structure co-design. Specifically, we apply framework dropout, gated bottlenecks, and hyperbolic cross attention that prevent the antibody shortcut path. In the decoder, a Mixture Density Network (MDN) sequence head with Potts-like pairwise coupling and annealed Multiple Choice Learning (aMCL) replaces the cross-entropy objective with a multi-component distribution. An antigen cycle consistency head routes gradients through the sequence decoder, forcing predicted distributions to encode antigen identity. On the \chimera{} dataset, \af{} leads the baselines on every interface metric on all three splits (fnat, iRMSD, DockQ, and epitope F1) together with structure and sequence, while achieving roughly a 2x larger effective vocabulary compared to the GNN baselines.

\end{abstract}

\section{Introduction}
\label{sec:intro}

Antibodies recognize and neutralize foreign antigens through their complementarity-determining regions (CDRs) that form the primary binding interface~\citep{potocnakova2016introduction}. Among these, CDR-H3 exhibits the greatest sequence and structural diversity and contributes most to antigen specificity~\citep{chothia1987canonical}. Designing CDR sequences and backbone structures that bind a target epitope is a central challenge in therapeutic antibody engineering~\citep{hummer2022advances}, and a growing body of deep generative models now address this problem by conditioning on antigen structure~\citep{luo2022diffab,kong2022mean,kong2023dymean,verma2023abode,wu2025raad,abir2025abflownet,tan2025dyab}. A natural expectation is that conditioning on the antigen should produce designs tailored to the target. Recent evidence suggests otherwise: unigram frequency alone explains most predictions~\citep{kong2023dymean}, BLOSUM substitution scores predict model outputs as well as learned likelihoods~\citep{uccar2025blosum}, and removing the antigen leaves predictions nearly unchanged~\citep{mason2025inversefoldingbenchmark}. These findings point to a systematic failure of existing conditioning mechanisms, but no prior work has identified the root causes or proposed principled remedies.

We diagnose three causally linked failure modes that explain why graph neural networks (GNNs) with greedy decoding fail to condition on antigen information for CDR design. \textbf{Antigen blindness}: all methods we evaluate show substantially lower recovery at antigen-contacting positions than non-contacting positions, and an antigen-free baseline achieves a better binding quality, confirming that existing conditioning mechanisms contribute little. \textbf{Vocabulary collapse}: GNN methods with greedy decoding produce effective vocabularies of three to five amino acids per position, far below native diversity, with rare but biochemically important residues almost never predicted. \textbf{Argmax bottleneck:} taking the argmax on cross-entropy loss throws the conditioning information away, and serves as a position-and-length null predictor that recovers 35\% of residues with limited diversity of amino acids, and no method in our comparison beats it by more than 5\% recovery on \chimera{}. 

\begin{figure*}[h!]
    \centering
    \includegraphics[width=0.63\linewidth]{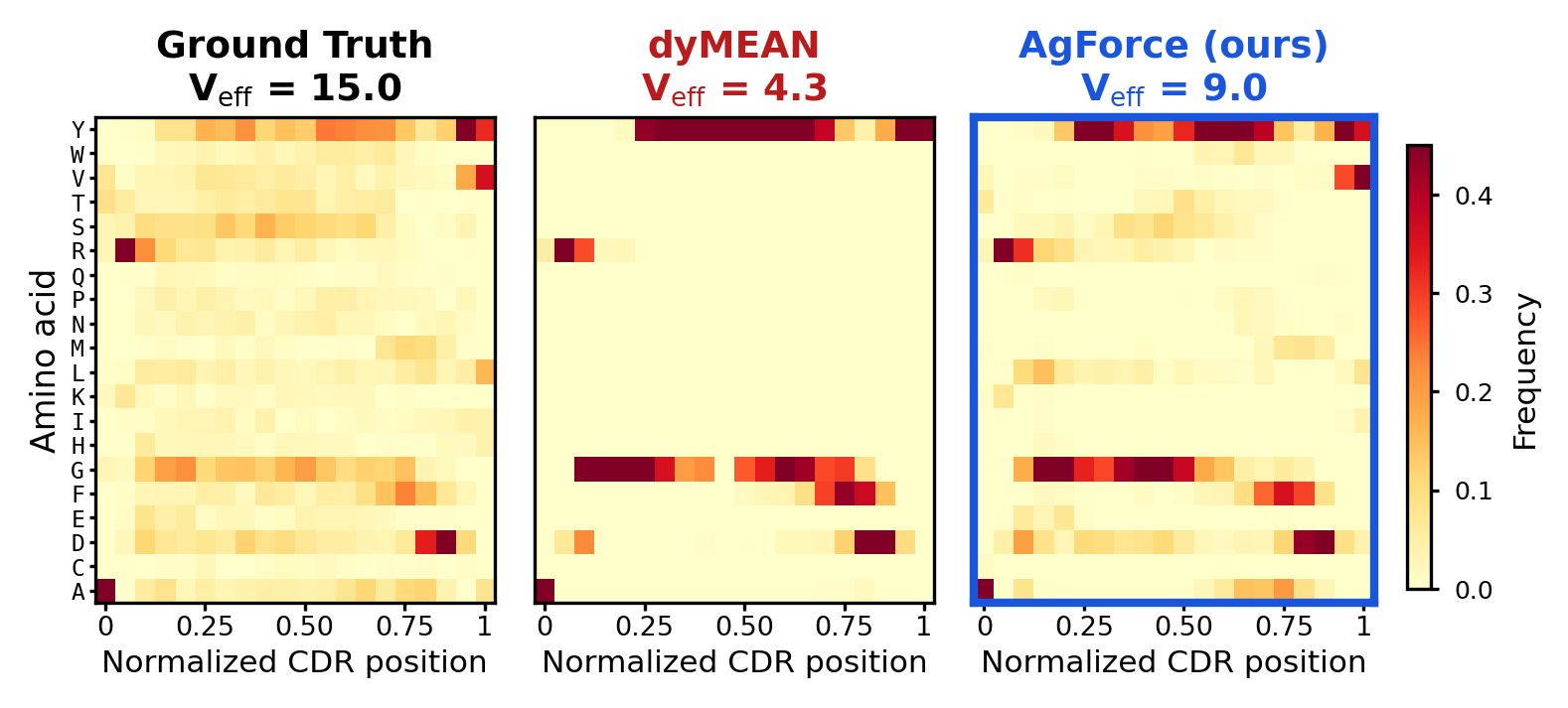}
    \includegraphics[width=0.36\linewidth]{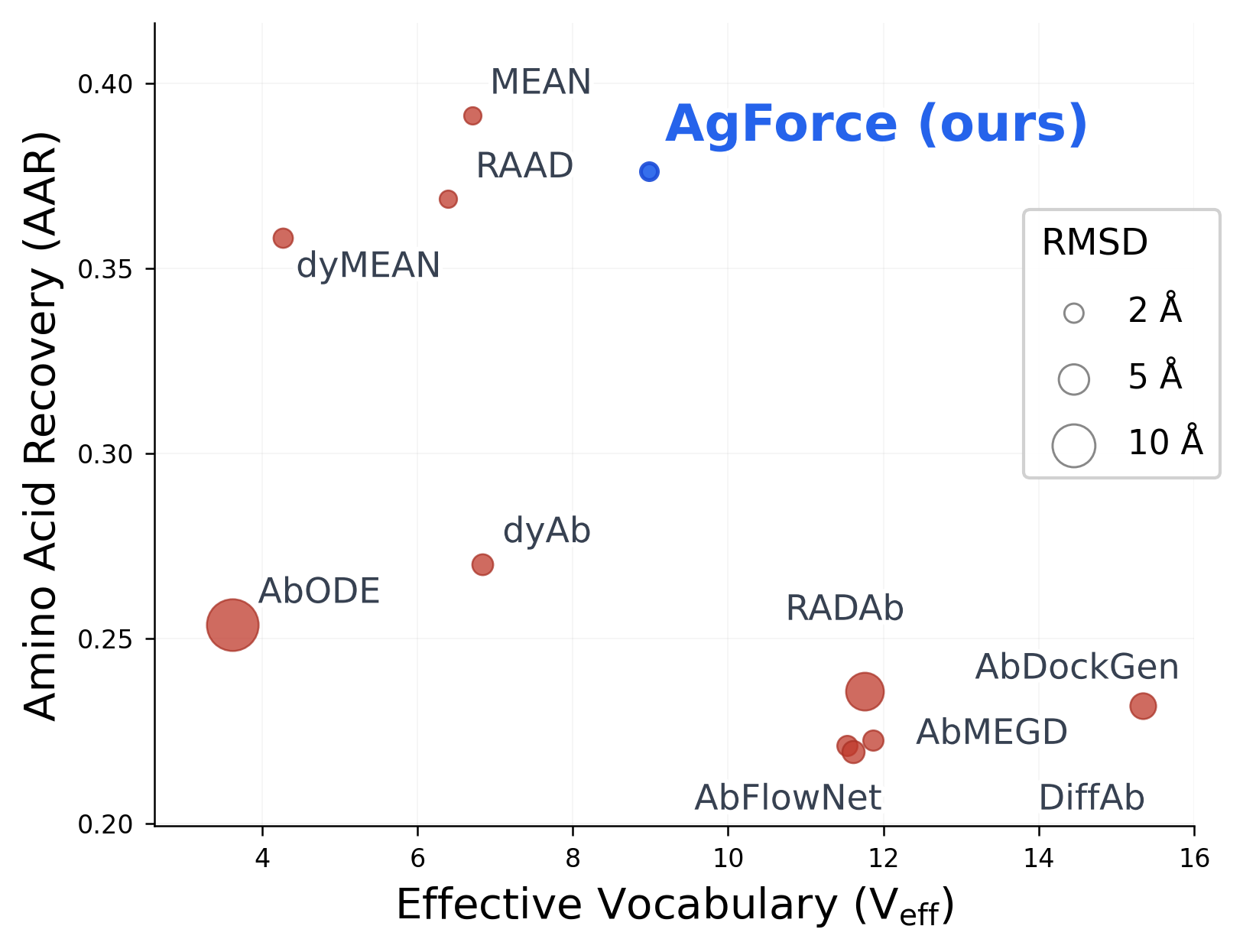}
    \caption{\textbf{(a) Vocabulary collapse:} Heatmap of the ground truth amino acid distribution of CDR H3, the predictions by dyMEAN, and our proposed model, \af{}. (b) AAR vs. effective vocabulary.}
    \label{fig:placeholder}
\end{figure*}

To address these failure modes, we introduce a novel encoder-decoder architecture called \af{}, that uses an equivariant graph neural network (EGNN) as the encoder and specialized decoders for sequence-structure co-design. The decoder employs a mixture density network (MDN) with Potts-like pairwise coupling, and is trained with annealed multiple choice learning (aMCL) so that several modes of the sequence posterior survive to decoding. Similarly, a framework dropout module corrupts the antibody context during training, and an antigen consistency head routes gradients through the sequence decoder to enforce antigen conditioning. For vocabulary collapse, the multi-component MDN permits diverse predictions across components, and Generative Determinantal Point Process (GDPP) spectral regularization penalizes deviations from native diversity. Our contributions:

\begin{enumerate}
    \item A \textbf{diagnostic framework} identifying three causally linked failure modes in equivariant GNN methods for CDR design, showing that greedy decoding strips antigen specificity from the output, and verified empirically across baselines spanning five architectural families.

    \item \textbf{\af{}}, which introduces targeted interventions for each failure: an MDN-Potts sequence head with aMCL training, framework dropout, and an antigen consistency loss, and virtual-node-augmented equivariant message passing with hyperbolic cross-attention. 
    
    \item On the \chimera{} dataset, \af{} improves on every interface and structural metric on all three splits, surpassing the baselines while producing more diverse sequence predictions than any GNN baseline.
\end{enumerate}

\section{Related Work}
\label{sec:related}

\textbf{Equivariant GNN methods.}
MEAN~\citep{kong2022mean} formulates CDR design as E(3)-equivariant graph translation, dyMEAN~\citep{kong2023dymean} extends this to full-atom design with a shadow paratope loss, and RAAD~\citep{wu2025raad} adds a contrastive specificity loss. Recent work on CDR sequence-structure co-design, such as EvoStruct~\cite{ahmed2026evostruct} and ConTact~\cite{ahmed2026contact}, has explored combining the strengths of protein language models (PLMs) with GNNs and performing geometric reasoning explicitly on the antigen-antibody contacts, respectively. 
These methods condition on antigen through spatial message passing and achieve the highest sequence recovery among all paradigms. Multiple independent studies corroborate the conditioning failure: BLOSUM substitution scores predict model outputs as accurately as learned likelihoods~\citep{uccar2025blosum}, removing the antigen chain leaves predictions nearly unchanged~\citep{mason2025inversefoldingbenchmark}, and a BLOSUM lookup table outperforms deep learning methods in binder enrichment~\citep{chinery2024simple}. While these studies document the problem, they fall short of identifying the root causes and proposing methodological interventions.

\textbf{Diffusion, flow, and ODE methods.}
DiffAb~\citep{luo2022diffab} models CDR generation as diffusion on SE(3) $\times$ categorical, while AbFlowNet~\citep{abir2025abflownet} replaces diffusion with GFlowNet trajectory balance. AbMEGD~\citep{chen2025AbMEGD} adds multi-scale encoding, dyAb~\citep{tan2025dyab} applies flow matching, AbODE~\citep{verma2023abode} uses conjoined ODEs, and RADAb~\citep{wang2024radab} augments diffusion with retrieval. These sampling-based methods maintain higher amino acid diversity but achieve substantially lower sequence recovery than GNN methods.

\textbf{Multi-modal sequence prediction and conditioning.}
TERMinator~\citep{li2023terminator} derives Potts energy tables from a GNN encoder for MCMC-based sequence design. PottsMPNN~\citep{birnbaum2026beyond} adds pairwise Potts supervision to ProteinMPNN, improving thermodynamic stability beyond what native sequence recovery captures. MDNs~\citep{bishop1994mixture} model multi-modal distributions but have not been applied to discrete amino acid prediction. Annealed multiple choice learning~\citep{lee2024amcl} addresses winner-take-all collapse in hypothesis ensembles but has not been applied to biological sequences. 

Our MDN-Potts head is built on the premise that greedy decoding cannot distinguish a conditioned model from an antigen-blind one. What a mixture must supply here is not expressivity but \textit{surviving modes}, which is why the coupling matrices are unrolled as differentiable, per-position gated belief propagation inside the head. The annealed multiple choice learning~\cite{lee2024amcl} is a constitutive regularizer where the components converge to the same mode and the head reduces to a single softmax. The same logic dictates where conditioning must be enforced. Classifier-free guidance~\cite{ho2022classifierfree} assumes a sampler whose output is a draw, not a mode, and RAAD's contrastive specificity loss~\cite{wu2025raad} constrains encoder embeddings that the sequence head is free to marginalize away. Our antigen-classification loss instead differentiates through the decoded probabilities themselves, so that no parameter between the antigen and the emitted residue escapes the gradient.


\section{Preliminaries}
\label{sec:preliminaries}

\subsection{Task Definition}
\label{sec:task}

We adopt the formulation from \chimera{}~\citep{ahmed2026chimerabench}. Given an antigen structure $A = \{(s_j, \mathbf{x}_j) \mid j \in V_A\}$, epitope $E \subseteq V_A$, and antibody framework $F = \{(s_i, \mathbf{x}_i) \mid i \in V_\text{FR}\}$, the task is to design CDR residues $R = \{(s_k, \mathbf{x}_k) \mid k \in V_\text{CDR}\}$ that maximize the conditional likelihood while satisfying epitope contact constraints:
\begin{gather}
    R^* = \argmax_{R} \; p_\theta\!\bigl(R \mid A, E, F\bigr), \\
    \text{s.t.} \;\; \mathcal{C}(R, A) \subseteq E, \;\; \mathcal{C}(R, A) \neq \emptyset
    \label{eq:task}
\end{gather}
where each residue is represented by its amino acid type $s_k \in \{1, \ldots, 20\}$, C$\alpha$ coordinate $\mathbf{x}_k \in \mathbb{R}^3$, and $\mathcal{C}(R, A) = \{j \in V_A \mid \exists\, k \in V_\text{CDR}: \|\mathbf{x}_k - \mathbf{x}_j\| < d_c\}$ is the set of antigen residues contacted by the designed CDRs within cutoff $d_c$. We focus on CDR-H3 as the most variable loop and primary determinant of antigen specificity~\citep{chothia1987canonical}.

\subsection{Notation and Graph Construction}
\label{sec:graph}

The antibody-antigen complex is represented as a heterogeneous graph $\mathcal{G} = (V, \mathcal{E})$ where $V = V_\text{HC} \cup V_\text{LC} \cup V_A \cup V_\text{glob} \cup V_\text{vn}$ contains residue nodes from the heavy chain, light chain, and antigen, three global delimiter tokens, and $N_\text{vn} = 3$ virtual nodes~\citep{sestak2024vn_egnn}. Each residue node $i$ carries an amino acid type $s_i \in \{1, \ldots, 20\}$ and four backbone atom coordinates $\mathbf{X}_i \in \mathbb{R}^{4 \times 3}$. The edge set $\mathcal{E}$ is partitioned into 10 typed subsets covering intra-chain (radial, sequential, KNN), inter-chain (radial, KNN), global-to-chain, and virtual-node-to-epitope/CDR connectivity (Table~\ref{tab:edge_types} in Appendix~\ref{app:graph}). Each edge carries a feature vector encoding edge type, relative position, pairwise distance RBFs, quaternion orientation, and local frame directions.

\subsection{Failure Modes}
\label{sec:diagnosis}

We re-train and evaluate CDR-H3 design methods on the \chimera{} dataset. The diagnostics are grouped around three failure modes.

\textbf{1. Antigen Blindness:}
\label{sec:diagnosis_blind}
If a model conditions on the antigen, its predictions should depend on antigen identity beyond what position alone dictates, but they do not. We correlate each method's empirical substitution matrix $P(\text{pred} \mid \text{true})$ against a positional PWM null built from the test set (Spearman, off-diagonal). Every method matches the null at $r = 0.59$ to $0.80$ (all $p < 10^{-35}$), and none exceeds $|r| = 0.12$ against the BLOSUM45 substitution matrix. The predictions are positional-frequency draws rather than evolutionary or binding-driven substitutions. Moreover, splitting amino acid recovery by antigen contact presents a similar picture from the binding side. For example, MEAN recovers 48.4\% at non-contact positions and 20.3\% at contacts.

\textbf{2. Vocabulary Collapse:}
\label{sec:diagnosis_vocab}
We measure predicted amino acid diversity via the effective vocabulary $\text{EV} = \exp\!\bigl(-\sum_a p(a) \log p(a)\bigr)$. The native CDR-H3 sequences exhibit EV $= 14.9$. Whereas the GNN methods with greedy decoding collapse to EV 3.7--6.9, overrepresenting tyrosine and glycine (MEAN by 25.2 and 7.2 percentage points, dyMEAN by 24.4 and 19.2) while rare interface-critical residues such as tryptophan, cysteine, and methionine are predicted at near-zero frequency. Motif diversity is correspondingly impoverished, with 93 unique bigrams for MEAN and 15 for dyMEAN against 353 in native sequences. Sampling-based methods approach native diversity (EV 11.6--15.3) but at the cost of accuracy (AAR 0.22--0.24). The ordering becomes apparent, since the six methods with the lowest EV are exactly the six with the highest AAR.

\textbf{3. Argmax Bottleneck:}
\label{sec:diagnosis_ce}
Standard CDR design methods minimize per-position cross-entropy $\mathcal{L}_\text{CE} = -\sum_{i} \log p_\theta(s_i \mid \mathbf{x})$ and decode by taking the argmax. Our analysis show that argmax decoding on the cross-entropy loss is one such cause of failure. Write $F_i$ for the antibody-side context at position $i$ (index, loop length, framework) and $A$ for the antigen.

\begin{proposition}[Conditioning Budget]
\label{prop:budget}
The expected per-position cross-entropy of any model is bounded below by $H(s_i \mid F_i, A)$, with equality only at the true conditional. The antigen-blind optimum is $H(s_i \mid F_i)$, so the entire cross-entropy reduction that antigen conditioning can capture equals the conditional mutual information $I(s_i; A \mid F_i)$.
\end{proposition}

\begin{proposition}[Argmax Identifiability]
\label{prop:identifiability}
Let $\mathcal{M} = \{i : \arg\max_a p^*(a \mid F_i, A) \neq \arg\max_a p^*(a \mid F_i)\}$ be the positions where the antigen moves the modal residue. A perfectly conditioned model and a perfectly antigen-blind model produce identical greedy decodes outside $\mathcal{M}$, so their amino acid recovery differs by at most $|\mathcal{M}| / L$.
\end{proposition}

\begin{corollary}[Vocabulary Collapse under Argmax Decoding]
\label{cor:vocab_collapse}
Greedy decoding emits a deterministic function of the conditioning, so the effective vocabulary of the output is bounded by the number of modes the model actually realizes, irrespective of how much entropy the underlying distribution holds.
\end{corollary}

The proofs are in Appendix~\ref{app:proofs}, while the empirical results are in Appendix~\ref{app:conditioning}.

\section{Method}
\label{sec:method}

\noindent
\begin{minipage}[t]{0.5\textwidth}
\vspace{0pt}
Figure~\ref{fig:architecture-overall} shows the overall pipeline. The encoder corrupts heavy-chain framework embeddings via dropout, encodes the complex with a EGNN, and computes CDR-to-epitope attention on the Lorentz hyperboloid. The decoder is an MDN-Potts head trained via annealed multiple choice learning. We detail each component below; training and inference algorithms are in Appendix~\ref{app:algorithms}.
\end{minipage}%
\hfill
\begin{minipage}[t]{0.45\textwidth}
\vspace{0pt}
\centering
\includegraphics[width=\linewidth]{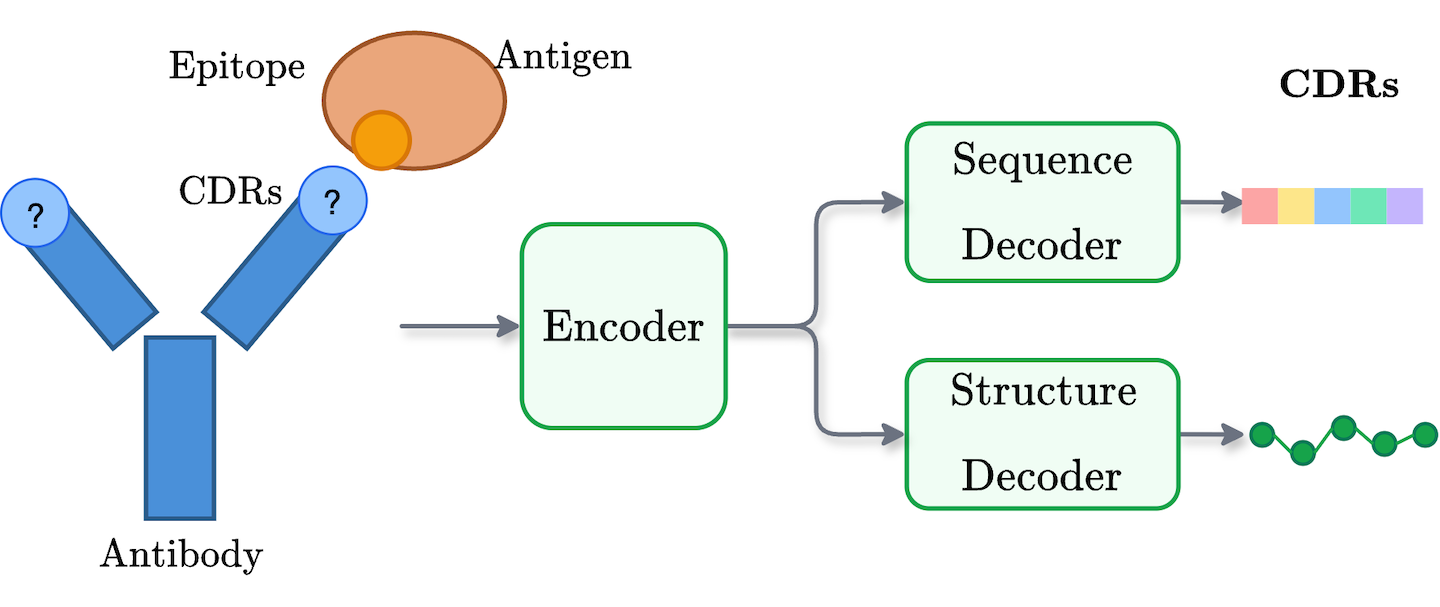}
\captionof{figure}{Overall pipeline of \af{}.}
\label{fig:architecture-overall}
\end{minipage}

\begin{figure*}[h!]
    \centering
    \includegraphics[width=0.75\linewidth]{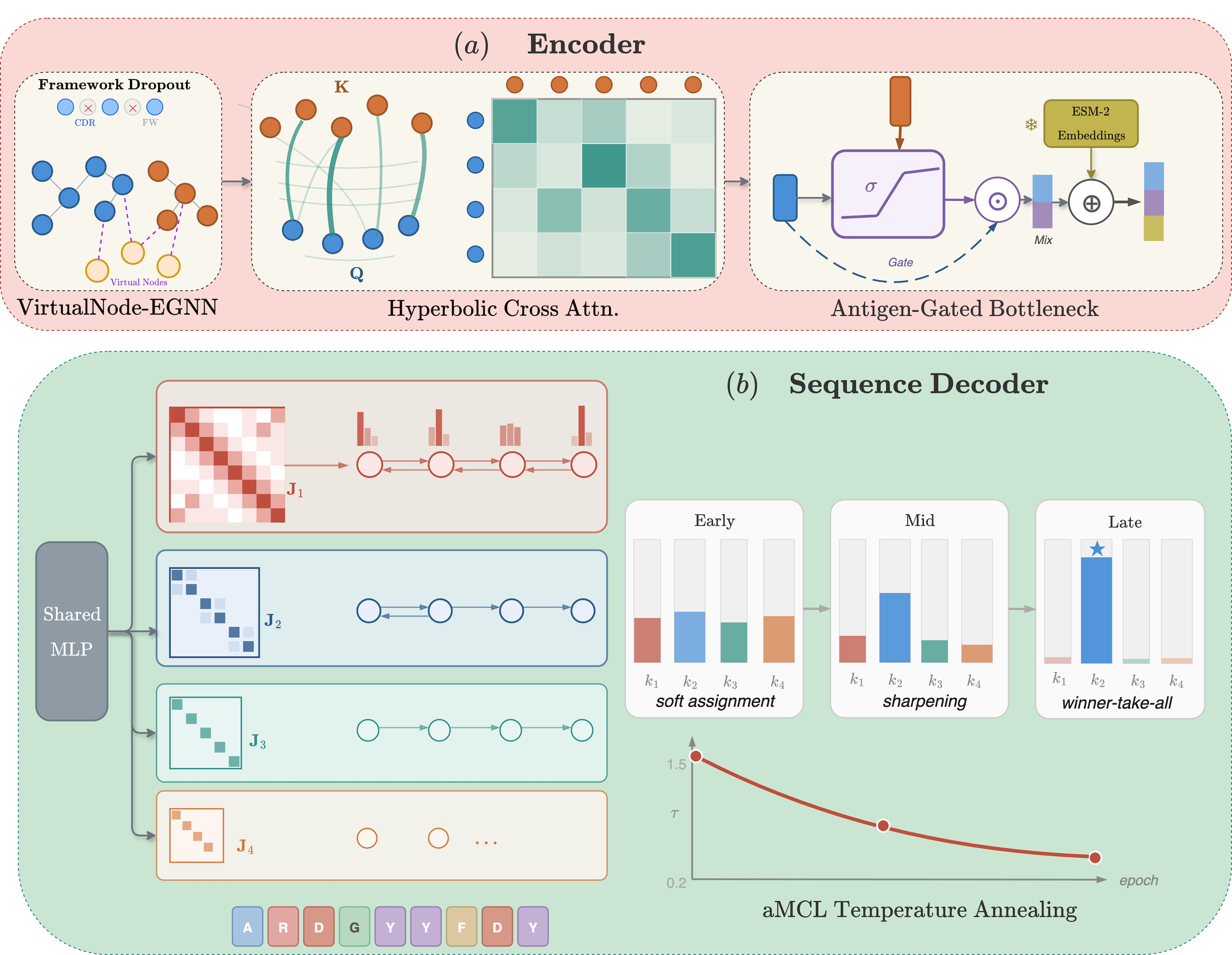}
    \caption{\af{} architecture. \textbf{Encoder}: Framework dropout corrupts heavy-chain framework embeddings before the EGNN encodes the complex. Hyperbolic cross-attention combines CDR and epitope embeddings with projected ESM-2 features. \textbf{Decoder}: The MDN-Potts head predicts amino acid distributions through $K=4$ mixture components with pairwise coupling, trained via annealed multiple choice learning (aMCL). }
    \label{fig:architecture}
\end{figure*}

\subsection{Feature Encoding and Framework Dropout}
\label{sec:features}

Each residue in the antibody-antigen complex is represented by a 108-dimensional feature vector encoding backbone geometry (sinusoidal position embeddings, bond distance RBFs, dihedral angles, local frame directions), amino acid identity (masked to zero for CDR positions during training), interface complementarity features, and a segment type embedding (full details in Appendix~\ref{app:graph}). A dual-path MLP processes geometric and chemical features through separate pathways with SiLU activations, fuses them, and projects to embedding dimension $d = 128$. Epitope residues receive an additional learnable embedding.

To address antigen blindness (Section~\ref{sec:diagnosis_blind}), we apply framework dropout before the GNN. During training, each heavy-chain framework residue embedding is independently set to zero with probability $p = 0.3$:
\begin{equation}
    \mathbf{h}_i^\text{train} = \begin{cases}
        \mathbf{0} & \text{if } i \in V_\text{FR}^\text{HC} \text{ and } m_i = 1 \\
        \mathbf{h}_i & \text{otherwise}
    \end{cases}, \quad m_i \sim \text{Bernoulli}(p)
\end{equation}
Because the corruption is applied \emph{before} message passing, its effect propagates through all subsequent GNN layers, forcing the model to extract information from the antigen pathway rather than relying on the antibody framework shortcut.

\subsection{Virtual-Node EGNN}
\label{sec:vn}

Operating on the graph $\mathcal{G}$, $N_\text{vn} = 3$ virtual nodes with learnable features and coordinates create a two-hop information channel between epitope and CDR residues, addressing the over-squashing problem~\citep{alon2021bottleneck} where information from distant epitope residues dilutes through sequential message passing.

The GNN consists of 5 relation-aware EGNN layers~\citep{satorras2021n,wu2025raad}. Each layer updates node features and coordinates simultaneously. For edge $(i, j)$ of type $t$, the message is:
\begin{equation}
    \mathbf{m}_{ij} = \text{MLP}_\text{msg}\!\left([\mathbf{h}_i,\; \mathbf{h}_j,\; \text{vec}(\Delta\mathbf{x}_i \Delta\mathbf{x}_j^\top),\; \mathbf{e}_{ij}]\right)
\end{equation}
where $\Delta\mathbf{x}_i = \mathbf{x}_i - \mathbf{x}_j$ and $\text{vec}(\cdot)$ flattens the outer product of coordinate differences into a radial feature. The Gram matrix entries are dot products of displacement vectors, which are invariant to rotation and translation. Node features are updated via per-type linear aggregation:
\begin{equation}
    \mathbf{h}_i' = \mathbf{h}_i + \text{MLP}_\text{node}\!\left(\left[\mathbf{h}_i,\; \sum_{t=0}^{9} \mathbf{W}_t \sum_{j \in \mathcal{N}_t(i)} \mathbf{m}_{ij}\right]\right)
\end{equation}
where $\mathbf{W}_t$ is a type-specific projection matrix and $\mathcal{N}_t(i)$ is the set of neighbors under edge type $t$. Coordinates are updated equivariantly:
\begin{equation}
    \mathbf{x}_i' = \mathbf{x}_i + \sum_{t} \text{mean}_{j \in \mathcal{N}_t(i)}\!\left(\Delta\mathbf{x}_{ij} \odot \text{MLP}_t^\text{coord}(\mathbf{m}_{ij})\right)
\end{equation}
The $\Delta\mathbf{x} \odot \text{scalar}$ structure ensures equivariance by construction~\citep{satorras2021n}. After 5 layers, the GNN produces residue embeddings $\mathbf{h} \in \mathbb{R}^{N \times 256}$ and updated coordinates $\mathbf{Z} \in \mathbb{R}^{N \times 4 \times 3}$.

\begin{proposition}[SE(3)-Equivariance of \af{}]
\label{prop:equivariance}
Coordinate predictions are SE(3)-equivariant and sequence predictions are SE(3)-invariant (proof in Appendix~\ref{app:proofs}).
\end{proposition}

The group is SE(3) by design. The quaternion orientation features and the residue local frames are built from cross products, which are pseudovectors: $(R\mathbf{u}) \times (R\mathbf{v}) = \det(R)\, R(\mathbf{u} \times \mathbf{v})$. A reflection will flip them. This is a correct behavior for proteins, whose L-amino acid backbones are chiral, and a reflection-invariant model would assign equal likelihood to a mirror-image complex that cannot exist.

\subsection{Hyperbolic Cross-Attention}
\label{sec:hyp_attn}

\af{} performs CDR-to-epitope cross-attention on the Lorentz hyperboloid~\citep{nickel2017poincare} rather than standard Euclidean attention. Hyperbolic spaces represent hierarchical relationships naturally because their volume grows exponentially with distance, matching the structure of antibody-antigen binding where global epitope geometry constrains local amino acid choices. Given CDR embeddings $\mathbf{h}_\text{cdr} \in \mathbb{R}^{L \times D}$ and epitope embeddings $\mathbf{h}_\text{epi} \in \mathbb{R}^{E \times D}$, queries and keys are projected into $H $ heads and mapped onto the hyperboloid by computing $x_0 = \sqrt{1/c + \|\mathbf{x}_{1:}\|^2}$ with curvature $c$. Attention scores use pairwise hyperbolic distances:
\begin{align}
    d_\mathcal{H}(\mathbf{q}, \mathbf{k}) &= \frac{1}{\sqrt{c}} \operatorname{arccosh}\!\bigl(-c \langle \mathbf{q}, \mathbf{k} \rangle_\mathcal{L}\bigr) \\
    \alpha_{ij} &= \operatorname{softmax}_j\!\left(\frac{-d_\mathcal{H}(\mathbf{q}_i, \mathbf{k}_j)}{\sqrt{D/H}}\right)
\end{align}
where $\langle \cdot, \cdot \rangle_\mathcal{L}$ denotes the Minkowski inner product $\langle \mathbf{q}, \mathbf{k} \rangle_\mathcal{L} = -q_0 k_0 + \sum_{d=1}^{D/H} q_d k_d$. The output is aggregated in Euclidean space via a weighted sum of value vectors.

\paragraph{Antigen gated bottleneck.} The attention output $\mathbf{o}_i$ passes through a gated bottleneck before reaching the sequence head:
\begin{equation}
    \mathbf{g}_i = \sigma(\mathbf{W}_g \mathbf{o}_i + \mathbf{b}_g), \quad
    \tilde{\mathbf{h}}_i = \alpha \cdot (\mathbf{h}_i^\text{cdr} \odot \mathbf{g}_i) + (1 - \alpha) \cdot \mathbf{h}_i^\text{cdr}
\end{equation}
where $\alpha = \sigma(\alpha_\text{logit}) \cdot 2\alpha_0$ is a learnable mixing coefficient initialized near $\alpha_0 = 0.5$. The output is the concatenation $[\tilde{\mathbf{h}}_i,\; \mathbf{o}_i] \in \mathbb{R}^{512}$. 

\subsection{Mixture Density Network (MDN) with Pairwise Coupling}
\label{sec:mdn}

To keep more than one mode alive through decoding, \af{} replaces the linear sequence head with an MDN. A shared layer applies layer normalization, a linear projection, SiLU activation, and dropout. From this, $K = 4$ component heads each produce logits $\boldsymbol{\ell}_k \in \mathbb{R}^{L \times 20}$, and a mixing head produces per-position weights $\boldsymbol{\pi} \in \mathbb{R}^{L \times K}$ via softmax.

Each component $k$ is equipped with a symmetric coupling matrix $\mathbf{J}_k \in \mathbb{R}^{20 \times 20}$~\citep{marks2011protein} that captures pairwise amino acid preferences between adjacent positions. The component logits are refined through two rounds of belief propagation:
\begin{align}
    \mathbf{m}_{i \to i+1} &= \boldsymbol{b}_i \cdot \frac{\mathbf{J}_k + \mathbf{J}_k^\top}{2}, \quad
    \mathbf{m}_{i+1 \to i} = \boldsymbol{b}_{i+1} \cdot \frac{\mathbf{J}_k + \mathbf{J}_k^\top}{2} \\
    \boldsymbol{\ell}_k^{(i)} &\leftarrow \boldsymbol{\ell}_k^{(i)} + g_i \cdot \bigl(\mathbf{m}_{i-1 \to i} + \mathbf{m}_{i+1 \to i}\bigr)
\end{align}
where $\boldsymbol{b}_i = \text{softmax}(\boldsymbol{\ell}_k^{(i)})$ are soft beliefs and $g_i = \sigma(\text{MLP}(\mathbf{h}_i))$ is a learned gate. After two rounds, the coupling matrices learn which amino acid pairs are compatible at adjacent CDR positions.

\paragraph{Annealed Multiple Choice Learning (aMCL).} We adopt aMCL~\citep{lee2024amcl} rather than standard mixture likelihood. The loss for each component combines cross-entropy and pairwise energy:
\begin{equation}
    \ell_k = \mathcal{L}_\text{CE}^{(k)} + \lambda_\text{pair} \cdot E_\text{pair}^{(k)}, \quad
    E_\text{pair}^{(k)} = \frac{1}{L{-}1} \sum_{i=1}^{L-1} \boldsymbol{b}_i^\top \frac{\mathbf{J}_k + \mathbf{J}_k^\top}{2} \boldsymbol{b}_{i+1}
\end{equation}
Components are assigned soft Boltzmann weights based on their losses:
\begin{equation}
    w_k = \frac{\exp(-\ell_k / \tau)}{\sum_{k'} \exp(-\ell_{k'} / \tau)}, \quad
    \mathcal{L}_\text{seq} = \sum_{k=1}^{K} w_k \cdot \ell_k
    \label{eq:amcl}
\end{equation}
Temperature $\tau$ anneals from $\tau_\text{start}$ to $\tau_\text{end}$ overtime with training.
At high temperature, all components receive roughly equal gradient and explore. As $\tau$ decreases, the assignment sharpens toward winner-take-all, encouraging specialization.

As $\tau \to 0$, the Boltzmann weights harden into the winner-take-all assignment of classical multiple choice learning, under which each component converges to the conditional distribution of the subset it wins~\citep{guzman2012mcl, cortes2025wta}. The mixture therefore covers several modes of the sequence posterior rather than one. During inference, the decoder selects a single component and then takes its argmax, so the component diversity raises the number of realized modes. 

\paragraph{GDPP diversity regularization.} We regularize the mixture with a Generative Determinantal Point Process (DGPP) loss~\citep{gdpp2019} that matches eigenvalue spectra of kernel matrices from predicted and true distributions:
\begin{equation}
    \mathcal{L}_\text{GDPP} = \left\|\text{eigenvalues}\!\left(\mathbf{K}_\text{pred}\right) - \text{eigenvalues}\!\left(\mathbf{K}_\text{true}\right)\right\|_2^2
\end{equation}
where $\mathbf{K}_\text{pred} = \mathbf{P}\mathbf{P}^\top + \epsilon \mathbf{I}$ with $\mathbf{P}$ the softmax probabilities from the active component.

\subsection{Antigen Classification Loss}
\label{sec:cycle}

This loss forces the predicted CDR probability distributions to encode antigen-specific information. Given a batch of $B$ complexes, let $\mathbf{p}_i \in \mathbb{R}^{L_i \times 20}$ denote the softmax probabilities from the sequence head for complex $i$, and let $\mathbf{a}_i \in \mathbb{R}^{D}$ denote the mean-pooled antigen embedding from the GNN. The classification head first computes a differentiable ``soft sequence'' embedding:
\begin{equation}
    \mathbf{c}_i = \text{MLP}\!\left(\frac{1}{L_i} \sum_{j=1}^{L_i} \mathbf{p}_i[j] \cdot \mathbf{E}_\text{AA}\right)
    \label{eq:soft_encode}
\end{equation}
where $\mathbf{E}_\text{AA} \in \mathbb{R}^{20 \times D}$ are learnable amino acid embeddings. The loss is an InfoNCE~\cite{chen2020simple} objective over the batch:
\begin{equation}
    \mathcal{L}_\text{cls} = -\frac{1}{B}\sum_{i=1}^{B} \log \frac{\exp(\mathbf{c}_i^\top \mathbf{a}_i)}{\sum_{k=1}^{B} \exp(\mathbf{c}_i^\top \mathbf{a}_k)}
    \label{eq:cycle_loss}
\end{equation}
If the model ignores the antigen and predicts identical CDRs for all complexes, $\mathbf{c}_i \approx \mathbf{c}_j$ and classification collapses to $1/B$. The key property is that gradients flow through the softmax output of the sequence head and back through the entire decoder pathway. This distinguishes it from embedding-level contrastive losses where the gradient bypasses the sequence decoder entirely.

\subsection{Joint Training Objective}
\label{sec:objective}

The full training objective combines five loss terms spanning structure prediction, sequence prediction, interface geometry, sequence diversity, and antigen forcing:
\begin{equation}
    \mathcal{L} = \mathcal{L}_\text{seq} + \alpha \mathcal{L}_\text{coord} + \delta \mathcal{L}_\text{shadow} + \epsilon \mathcal{L}_\text{GDPP} + \lambda_\text{cls} \mathcal{L}_\text{cls}
    \label{eq:full_loss}
\end{equation}
The coordinate loss $\mathcal{L}_\text{coord}$ is a smooth-$\ell_1$ (Huber) loss on predicted versus true C$\alpha$ backbone coordinates for CDR positions:
\begin{equation}
    \mathcal{L}_\text{coord} = \sum_{k \in V_\text{CDR}} \text{smooth}_{\ell_1}\!\left(\mathbf{Z}_k - \mathbf{x}_k^\text{true}\right)
\end{equation}
The shadow paratope loss $\mathcal{L}_\text{shadow}$~\citep{kong2023dymean} enforces that the predicted CDR backbone maintains correct distances to the epitope surface:
\begin{equation}
    \mathcal{L}_\text{shadow} = \frac{1}{|\mathcal{E}_\text{epi}| \cdot L} \sum_{j \in \mathcal{E}_\text{epi}} \sum_{k \in V_\text{CDR}} \left| \|\hat{\mathbf{x}}_k - \mathbf{x}_j\| - \|\mathbf{x}_k^\text{true} - \mathbf{x}_j\| \right|
\end{equation}
where $\mathcal{E}_\text{epi}$ is the set of epitope residues within 8~\AA{} of any CDR residue. 
At inference, greedy decoding selects the highest-weight mixture component per position and the argmax amino acid.

\section{Experiments}
\label{sec:experiments}

\subsection{Setup}

\paragraph{Dataset and metrics.} We evaluate \af{} on \chimera{}~\citep{ahmed2026chimerabench}, comprising 2,922 antibody-antigen complexes. We report CDR-H3 on the antigen-fold split (2338/292/292 train/val/test), which holds out whole antigen folds and therefore measures generalization to unseen target families. The sequence quality is measured by amino acid recovery (AAR) and perplexity (PPL). Whereas the structure quality is measured by C$\alpha$ RMSD. Binding quality is measured by fnat, iRMSD, DockQ~\citep{basu2016dockq}, and epitope F1.

\paragraph{Baselines.} We compare \af{} against: equivariant GNNs (RAAD~\citep{wu2025raad}, MEAN~\citep{kong2022mean}, dyMEAN~\citep{kong2023dymean}), diffusion/flow models (DiffAb~\citep{luo2022diffab}, AbFlowNet~\citep{abir2025abflownet}, AbMEGD~\citep{chen2025AbMEGD}, RADAb~\citep{wang2024radab}, dyAb~\citep{tan2025dyab}), AbODE~\citep{verma2023abode}, and AbDockGen~\citep{jin2022hern}.

\subsection{Main Results}

\begin{table*}[h!]
\centering
\caption{CDR-H3 design on \chimera{} antigen-fold split. Best values are in \textbf{bold}, second-best \underline{underlined}. }
\label{tab:main_results}
\small
\setlength{\tabcolsep}{3pt}
\begin{tabular}{llccccccc}
\toprule
& Method & AAR$\uparrow$ & PPL$\downarrow$ & RMSD$\downarrow$ & fnat$\uparrow$ & iRMSD$\downarrow$ & DockQ$\uparrow$ & epiF1$\uparrow$ \\
\midrule
\multirow{3}{*}{\rotatebox{90}{\scriptsize GNN}}
& RAAD & \underline{0.38}\std{.13} & \underline{3.06}\std{.49} & \underline{1.65}\std{.70} & \underline{0.62}\std{.31} & \underline{1.42}\std{.62} & \underline{0.70}\std{.21} & \underline{0.78}\std{.25} \\
& MEAN & \textbf{0.40}\std{.14} & \textbf{2.79\std{.57}} & 1.66\std{.64} & \underline{0.62}\std{.32} & \underline{1.42}\std{.61} & \underline{0.70}\std{.21} & \underline{0.78}\std{.26} \\
& dyMEAN & 0.37\std{.13} & 3.35\std{.66} & 2.06\std{.83} & 0.61\std{.31} & 1.79\std{.82} & 0.66\std{.20} & 0.75\std{.25} \\
\midrule
\multirow{5}{*}{\rotatebox{90}{\scriptsize Diff./Flow}}
& DiffAb & 0.23\std{.12} & -- & 2.24\std{1.02} & 0.52\std{.34} & 2.10\std{.89} & 0.58\std{.25} & 0.65\std{.30} \\
& AbFlowNet & 0.23\std{.12} & -- & 2.70\std{5.36} & 0.56\std{.34} & 2.46\std{4.31} & 0.60\std{.26} & 0.65\std{.30} \\
& AbMEGD & 0.24\std{.13} & -- & 2.25\std{1.06} & 0.55\std{.34} & 2.13\std{.94} & 0.59\std{.25} & 0.65\std{.30} \\
& RADAb & 0.25\std{.13} & -- & 7.79\std{41.9} & 0.51\std{.33} & 4.78\std{22.8} & 0.57\std{.26} & 0.65\std{.31} \\
& dyAb & 0.28\std{.11} & -- & 2.40\std{.81} & 0.56\std{.33} & 1.81\std{.81} & 0.65\std{.21} & 0.72\std{.29} \\
\midrule
\rotatebox{90}{\scriptsize ODE}
& AbODE & 0.27\std{.13} & 10.78\std{3.13} & 13.17\std{4.81} & 0.16\std{.22} & 4.00\std{1.79} & 0.39\std{.16} & 0.40\std{.36} \\
\midrule
\rotatebox{90}{\scriptsize AR}
& AbDockGen & 0.24\std{.13} & 7.89\std{3.45} & 3.65\std{.81} & 0.46\std{.27} & 2.41\std{.84} & 0.57\std{.17} & 0.75\std{.23} \\
\midrule
& \textbf{\af{} (\textit{ours})} & \underline{0.38}\std{.14} & 3.17\std{.47} & \textbf{1.55\std{.67}} & \textbf{0.69\std{.29}} & \textbf{1.27\std{.57}} & \textbf{0.74\std{.21}} & \textbf{0.82\std{.22}} \\
\bottomrule
\end{tabular}
\end{table*}

Table~\ref{tab:main_results} presents the main comparison, where \af{} wins on almost all metrics. It attains the lowest RMSD (1.55~\AA) and the best binding quality on every interface measure: fnat 0.69 against 0.62 for the next best method, iRMSD 1.27~\AA{} against 1.42~\AA, DockQ 0.74 against 0.70, and epitope F1 0.82 against 0.78. \af{} produces consistent gains on structural and binding metrics rather than trading off, which is the pattern prior methods fail to achieve.

Moreover, \af{} performs competitively on AAR or perplexity. As noted earlier, the AAR gap measures vocabulary collapse, and not necessarily design quality. MEAN's effective vocabulary is 6.7 against a native 14.9, whereas \af{} holds 8.9, and Section~\ref{sec:diagnosis_ce} shows that recovery under argmax decoding is maximized by collapsing onto positional modes. A method that collapses harder scores higher AAR while producing less antigen-specific sequence. \af{} gives up 2 points of AAR for a 7-point fnat gain and a 4-point epitope F1 gain over MEAN.

The same ordering holds on the temporal split, where \af{} ties MEAN on AAR (0.39) and leads on perplexity, RMSD, fnat, iRMSD, DockQ, and epitope F1, while on the epitope-group split, it leads on perplexity, RMSD, and all four interface metrics. The per-CDR results and the antigen conditioning analysis in Appendix~\ref{app:conditioning}.

\paragraph{Discussion.} We note two observations from Table~\ref{tab:main_results}. First, the sequence and binding axes are close to independent. MEAN tops AAR while being fourth on fnat, and AbODE reaches AAR 0.27 with fnat 0.16. Sequence recovery therefore cannot serve as a proxy for binding quality, and a method should be judged on both. Second, the three GNN baselines span AAR 0.37 to 0.40 with nearly identical fnat (0.61 to 0.62), so their architectural differences move the sequence predictions slightly, while leaving the interface untouched. \af{} improves the interface instead by modeling interface geometry and antigen specificity directly rather than relying on message passing to propagate antigen information implicitly. The antigen-fold split holds out whole antigen folds, which makes these gains generalizable against targets structurally unlike anything in training.

\begin{table*}[h!]
\centering
\caption{Ablation study on CDR-H3 on epitope-group split. Each row either adds or removes one component of the full configuration. EV represents the effective vocabulary. The full \af{} model is from the epitope-group split in Table~\ref{tab:other_splits}. }
\label{tab:ablations}
\small
\setlength{\tabcolsep}{2.6pt}
\begin{tabular}{lcccccccc}
\toprule
Configuration & AAR$\uparrow$ & EV$\uparrow$ & PPL$\downarrow$ & RMSD$\downarrow$ & fnat$\uparrow$ & iRMSD$\downarrow$ & DockQ$\uparrow$ & epiF1$\uparrow$ \\
\midrule
\af{} (full) & 0.392 & 8.06 & 3.011 & 1.864 & 0.530 & 1.543 & 0.656 & 0.683 \\
\midrule
$-$ Hyperbolic attn. & 0.387 & 8.73 & 2.867 & 1.876 & 0.540 & 1.568 & 0.657 & 0.714 \\
$-$ Antigen cls loss & 0.402 & 8.06 & 3.029 & 1.905 & 0.537 & 1.597 & 0.653 & 0.677 \\
$-$ Shadow paratope & 0.390 & 8.29 & 3.025 & 1.880 & 0.524 & 1.579 & 0.651 & 0.698 \\
$-$ Framework dropout & 0.403 & 7.56 & 2.905 & 1.869 & 0.556 & 1.554 & 0.664 & 0.706 \\
$-$ aMCL & 0.390 & 8.39 & 3.277 & 1.899 & 0.543 & 1.589 & 0.656 & 0.700 \\
$-$ GDPP & 0.397 & 7.26 & 3.000 & 1.859 & 0.545 & 1.555 & 0.660 & 0.700 \\
\midrule
+ Plain MLP head & 0.394 & 7.89 & 3.059 & 1.853 & 0.539 & 1.548 & 0.659 & 0.701 \\
$-$ Pairwise term & 0.397 & 8.16 & 2.981 & 1.871 & 0.533 & 1.557 & 0.656 & 0.693 \\
$-$ MDN term & 0.396 & 8.18 & 2.984 & 1.884 & 0.547 & 1.563 & 0.661 & 0.697 \\
\midrule
$-$ ESM-2 & 0.381 & 7.40 & 3.180 & 1.894 & 0.543 & 1.581 & 0.657 & 0.688 \\
+ EGNN 3 layers & 0.393 & 7.59 & 3.023 & 1.956 & 0.501 & 1.641 & 0.637 & 0.665 \\
+ EGNN 7 layers & 0.393 & 7.98 & 3.034 & 1.834 & 0.562 & 1.530 & 0.668 & 0.707 \\
\bottomrule
\end{tabular}
\end{table*}

\subsection{Ablation Study}
\label{sec:ablation}

We isolate the contribution of each component by performing ablation experiments in Table~\ref{tab:ablations} on the epitope-group split by changing one component at a time in the full configuration. 
The depth of EGNN is a structural lever that helps \af{} learn better representations. 
The 3-layer variant loses 2.9 points of fnat, 0.10~\AA{} of iRMSD, and 0.09~\AA{} of C$\alpha$ RMSD, and the 7-layer variant recovers all three. The 0.12~\AA{} RMSD spread between the shallow and deep variants is six times the 0.02~\AA{} gap between our two full-model runs, which makes depth the one component whose structural effect is unambiguous. 
Everything else moves RMSD by less than 0.05~\AA{} and fnat by less than 0.03.

Two ablations \emph{raise} AAR above the full model, and effective vocabulary (EV) separates them. Removing framework dropout gives the highest AAR (0.403) alongside the second-lowest effective vocabulary (7.56), which is what the collapse mechanism of Section~\ref{sec:diagnosis_ce} predicts. 
Removing the antigen-classification loss raises AAR to 0.402 with effective vocabulary unchanged at 8.06, and improves the conditioning metrics. 
Moreover, the GDPP regularizer contributes to producing diverse sequence predictions with an EV improvement from 7.26 to 8.06.  
A plain MLP head matches the full MDN-Potts head on every column (AAR 0.394 versus 0.392, RMSD 1.853 versus 1.864~\AA), and replacing hyperbolic attention with Euclidean attention leaves fnat and DockQ flat while \emph{raising} effective vocabulary to 8.73, which is the highest in the table by bringing a slight decrement to the structural and interface metrics. 
The removal of frozen ESM-2 embeddings reduces AAR by 1.1 points. 
These experiments show that interface gains therefore come jointly from the interface modeling and the specialized decoding mechanism.

\textbf{Limitations.} 
The effective vocabulary, while larger than that of any GNN baseline, still reaches only 53\% of native diversity. The binding positions remain a fundamentally harder problem that may require explicit modeling of side-chain rotamers or physics-based energy terms.
\af{} currently generates one CDR at a time, and could be extended with joint multi-CDR generation. All evaluations are computational, and experimental validation of the designed sequences through binding assays would strengthen the conclusions.

\section{Conclusion}
\label{sec:conclusion}

We showed that equivariant GNN methods for antibody CDR design share three causally linked failure modes, which in turn produce antigen blindness and vocabulary collapse. We then presented \af{} that addresses each failure with a targeted intervention and leads every interface and structural metric on all three \chimera{} splits while holding more vocabulary than other comparable GNN baselines.

\newpage
\bibliography{references}

@article{hummer2022advances,
  title={Advances in computational structure-based antibody design},
  author={Hummer, Alissa and Abanades, Brennan and Deane, Charlotte},
  journal={Current Opinion in Structural Biology},
  volume={74},
  pages={102379},
  year={2022},
  publisher={Elsevier}
}

@article{wang2024radab,
  title={Retrieval augmented diffusion model for structure-informed antibody design and optimization},
  author={Wang, Zichen and Ji, Yaokun and Tian, Jianing and Zheng, Shuangjia},
  journal={arXiv preprint arXiv:2410.15040},
  year={2024}
}

@inproceedings{verma2023abode,
  title={Ab{ODE}: {A}b initio Antibody Design using Conjoined {ODE}s},
  author={Verma, Yogesh and Heinonen, Markus and Garg, Vikas},
  booktitle={International Conference on Machine Learning},
  pages={35037--35050},
  year={2023},
  organization={PMLR}
}

@article{potocnakova2016introduction,
  title={An introduction to B-cell epitope mapping and in silico epitope prediction},
  author={Potocnakova, Lenka and Bhide, Mangesh and Pulzova, Lucia Borszekova},
  journal={Journal of immunology research},
  volume={2016},
  year={2016},
  publisher={Hindawi}
}

@inproceedings{satorras2021n,
  title={E(n) equivariant graph neural networks},
  author={Satorras, V{\i}ctor Garcia and Hoogeboom, Emiel and Welling, Max},
  booktitle={International conference on machine learning},
  pages={9323--9332},
  year={2021},
  organization={PMLR}
}

@inproceedings{wu2025raad,
  title={Relation-aware equivariant graph networks for epitope-unknown antibody design and specificity optimization},
  author={Wu, Lirong and Lin, Haitao and Huang, Yufei and Gao, Zhangyang and Tan, Cheng and Liu, Yunfan and Wu, Tailin and Li, Stan Z},
  booktitle={Proceedings of the AAAI Conference on Artificial Intelligence},
  volume={39},
  number={1},
  pages={895--904},
  year={2025}
}

@article{sestak2024vn_egnn,
  title={{VN-EGNN}: {E(3)}-Equivariant Graph Neural Networks with Virtual Nodes Enhance Protein Binding Site Identification},
  author={Sestak, Florian and Schneckenreiter, Lisa and Brandstetter, Johannes and Hochreiter, Sepp and Mayr, Andreas and Klambauer, G{\"u}nter},
  journal={Journal of Cheminformatics},
  volume={18},
  pages={11},
  year={2026},
  doi={10.1186/s13321-025-01127-9}
}

@inproceedings{jin2022hern,
  title={Antibody-antigen docking and design via hierarchical structure refinement},
  author={Jin, Wengong and Barzilay, Regina and Jaakkola, Tommi},
  booktitle={International Conference on Machine Learning},
  pages={10217--10227},
  year={2022},
  organization={PMLR}
}

@inproceedings{chen2020simple,
  title={A simple framework for contrastive learning of visual representations},
  author={Chen, Ting and Kornblith, Simon and Norouzi, Mohammad and Hinton, Geoffrey},
  booktitle={International conference on machine learning},
  pages={1597--1607},
  year={2020},
  organization={PmLR}
}

@inproceedings{kong2023dymean,
  title={End-to-End Full-Atom Antibody Design},
  author={Kong, Xiangzhe and Huang, Wenbing and Liu, Yang},
  booktitle={International Conference on Machine Learning},
  pages={17409--17429},
  year={2023},
  organization={PMLR}
}

@article{luo2022diffab,
  title={Antigen-specific antibody design and optimization with diffusion-based generative models for protein structures},
  author={Luo, Shitong and Su, Yufeng and Peng, Xingang and Wang, Sheng and Peng, Jian and Ma, Jianzhu},
  journal={Advances in Neural Information Processing Systems},
  volume={35},
  pages={9754--9767},
  year={2022}
}

@inproceedings{tan2025dyab,
  title={dyAb: Flow Matching for Flexible Antibody Design with {A}lphaFold-driven Pre-binding Antigen},
  author={Tan, Cheng and Zhang, Yijie and Gao, Zhangyang and Huang, Yufei and Lin, Haitao and Wu, Lirong and Wu, Fandi and Blanchette, Mathieu and Li, Stan Z},
  booktitle={Proceedings of the AAAI Conference on Artificial Intelligence},
  volume={39},
  number={1},
  pages={782--790},
  year={2025}
}

@inproceedings{kong2022mean,
  title={Conditional Antibody Design as {3D} Equivariant Graph Translation},
  author={Kong, Xiangzhe and Huang, Wenbing and Liu, Yang},
  booktitle={International Conference on Learning Representations},
  year={2023}
}

@article{abir2025abflownet,
  title={Ab{F}low{N}et: Optimizing Antibody-Antigen Binding Energy via {D}iffusion-{GF}low{N}et Fusion},
  author={Abir, Abrar Rahman and Shahgir, Haz Sameen and Ratul, Md Rownok Zahan and Tahmid, Md Toki and Steeg, Greg Ver and Dong, Yue},
  journal={arXiv preprint arXiv:2505.12358},
  year={2025}
}

@article{chen2025AbMEGD,
  title={Antibody Design and Optimization with Multi-scale Equivariant Graph Diffusion Models for Accurate Complex Antigen Binding},
  author={Chen, Jiameng and Cai, Xiantao and Wu, Jia and Hu, Wenbin},
  journal={arXiv preprint arXiv:2506.20957},
  year={2025}
}

@article{basu2016dockq,
  title={Dock{Q}: a quality measure for protein-protein docking models},
  author={Basu, Sankar and Wallner, Bj{\"o}rn},
  journal={PloS one},
  volume={11},
  number={8},
  pages={e0161879},
  year={2016},
  publisher={Public Library of Science San Francisco, CA USA}
}

@article{chothia1987canonical,
  title={Canonical structures for the hypervariable regions of immunoglobulins},
  author={Chothia, Cyrus and Lesk, Arthur M},
  journal={Journal of molecular biology},
  volume={196},
  number={4},
  pages={901--917},
  year={1987},
  publisher={Elsevier}
}

@article{uccar2025blosum,
  title={{BLOSUM} Is All You Learn—Generative Antibody Models Reflect Evolutionary Priors},
  author={U{\c{c}}ar, Talip and Sormanni, Pietro},
  journal={bioRxiv},
  pages={2025--10},
  year={2025},
  publisher={Cold Spring Harbor Laboratory}
}

@inproceedings{gdpp2019,
  title={{GDPP}: Learning Diverse Generations Using Determinantal Point Processes},
  author={Elfeki, Mohamed and Couprie, Camille and Riviere, Morgane and Elhoseiny, Mohamed},
  booktitle={International Conference on Machine Learning},
  volume={97},
  pages={1819--1828},
  year={2019},
  organization={PMLR}
}

@inproceedings{lee2024amcl,
  title={Annealed Multiple Choice Learning: Overcoming limitations of Winner-takes-all with annealing},
  author={Perera, David and Letzelter, Victor and Mariotte, Th{\'e}o and Cort{\'e}s, Adrien and Chen, Mickael and Essid, Slim and Richard, Ga{\"e}l},
  booktitle={Advances in Neural Information Processing Systems},
  volume={37},
  year={2024}
}

@inproceedings{alon2021bottleneck,
  title={On the Bottleneck of Graph Neural Networks and its Practical Implications},
  author={Alon, Uri and Yahav, Eran},
  booktitle={International Conference on Learning Representations},
  year={2021}
}

@inproceedings{nickel2017poincare,
  title={Poincar{\'e} Embeddings for Learning Hierarchical Representations},
  author={Nickel, Maximillian and Kiela, Douwe},
  booktitle={Advances in Neural Information Processing Systems},
  volume={30},
  year={2017}
}

@article{marks2011protein,
  title={Protein {3D} Structure Computed from Evolutionary Sequence Variation},
  author={Marks, Debora S. and Colwell, Lucy J. and Sheridan, Robert and Hopf, Thomas A. and Pagnani, Andrea and Zecchina, Riccardo and Sander, Chris},
  journal={PLoS ONE},
  volume={6},
  number={12},
  pages={e28766},
  year={2011},
  doi={10.1371/journal.pone.0028766}
}

@techreport{bishop1994mixture,
  title={Mixture Density Networks},
  author={Bishop, Christopher M.},
  year={1994},
  institution={Aston University},
  number={NCRG/94/004},
  address={Birmingham, UK}
}

@article{li2023terminator,
  title={Neural network-derived {P}otts models for structure-based protein design using backbone atomic coordinates and tertiary motifs},
  author={Li, Alex J. and Sundar, Vikram and Grigoryan, Gevorg and Keating, Amy E.},
  journal={Protein Science},
  volume={32},
  number={2},
  pages={e4554},
  year={2023},
  doi={10.1002/pro.4554}
}

@article{birnbaum2026beyond,
  title={Beyond native sequence recovery: Improved modeling of the sequence-energy landscape of protein structures},
  author={Birnbaum, Foster and Keating, Amy E},
  journal={bioRxiv},
  pages={2026--01},
  year={2026},
  publisher={Cold Spring Harbor Laboratory}
}

@article{chinery2024simple,
  title={Simple computational methods can outperform deep learning in designing diverse, binder-enriched antibody libraries},
  author={Chinery, Lewis and Hummer, Alissa M and Mehta, Brij Bhushan and Akbar, Rahmad and Rawat, Puneet and Slabodkin, Andrei and Le Quy, Khang and Lund-Johansen, Fridtjof and Greiff, Victor and Jeliazkov, Jeliazko R and Deane, Charlotte M},
  journal={bioRxiv},
  year={2024},
  doi={10.1101/2024.03.26.586756}
}

@article{ho2022classifierfree,
  title={Classifier-Free Diffusion Guidance},
  author={Ho, Jonathan and Salimans, Tim},
  journal={arXiv preprint arXiv:2207.12598},
  year={2022}
}

@article{cortes2025wta,
  title={Winner-takes-all for multivariate probabilistic time series forecasting},
  author={Cort{\'e}s, Adrien and Rehm, R{\'e}mi and Letzelter, Victor},
  journal={arXiv preprint arXiv:2506.05515},
  year={2025}
}

@article{guzman2012mcl,
  title={Multiple choice learning: Learning to produce multiple structured outputs},
  author={Guzman-Rivera, Abner and Batra, Dhruv and Kohli, Pushmeet},
  journal={Advances in neural information processing systems},
  volume={25},
  year={2012}
}

@article{mason2025inversefoldingbenchmark,
  title={Benchmarking inverse folding models for antibody {CDR} sequence design},
  author={Li, Yifan and Lang, Yuxiang and Xu, Chenrui and Zhou, Yi and Pang, Ziwei and Greisen, Per Jr},
  journal={PLOS ONE},
  volume={20},
  number={6},
  pages={e0324566},
  year={2025},
  doi={10.1371/journal.pone.0324566}
}

@inproceedings{ahmed2026evostruct,
  title={EvoStruct: Bridging Evolutionary and Structural Priors for Antibody CDR Design via Protein Language Model Adaptation},
  author={Ahmed, Mansoor and Lee, Sujin and Khayaz, Umar and Patterson, Murray},
  booktitle={The 2026 Workshop on Generative and Agentic AI for Biology},
  year={2026}
}

@article{ahmed2026contact,
  title={ConTact: Contact-First Antibody CDR Design via Explicit Interface Reasoning},
  author={Ahmed, Mansoor and VonBank, Spencer and Taj, Nadeem and Lee, Sujin and Jan, Naila and Patterson, Murray},
  journal={arXiv preprint arXiv:2605.21600},
  year={2026}
}

@inproceedings{
ahmed2026chimerabench,
title={{C}hi{MER}a-Bench: A Benchmark Dataset for Epitope-Specific Antibody Design},
author={Mansoor Ahmed and Nadeem Taj and Imdad Ullah Khan and Hemanth Venkateswara and Murray Patterson},
booktitle={ICLR 2026 Workshop on Generative and Experimental Perspectives for Biomolecular Design},
year={2026}
}

\onecolumn
\appendix

\section{Appendix}



\subsection{Proofs}
\label{app:proofs}

\begin{snugshade}\noindent
\textbf{Proposition~\ref{prop:budget}} (Conditioning Budget)
Fix a position $i$ in the CDR loop. Let $s_i \in \mathcal{A}$ be the native residue there, over the alphabet of $|\mathcal{A}| = 20$ amino acids. Let $F_i$ collect the antibody-side context available at $i$ (position index, loop length, framework sequence) and let $A$ denote the antigen. Let $p^*$ be the joint law of $(s_i, F_i, A)$ induced by the data distribution, and consider two families of predictors,
\begin{equation*}
\mathcal{P}_\mathrm{cond} = \bigl\{ q(\cdot \mid F_i, A) \bigr\},
\qquad
\mathcal{P}_\mathrm{blind} = \bigl\{ q(\cdot \mid F_i) \bigr\} \subset \mathcal{P}_\mathrm{cond},
\end{equation*}
the second being those predictors measurable with respect to $F_i$ alone. Define the population cross-entropy risk $\mathcal{R}(q) = \mathbb{E}\bigl[-\log q(s_i \mid \cdot)\bigr]$. Then:
\begin{enumerate}
\item[(i)] $\mathcal{R}(q) \ge H(s_i \mid F_i, A)$ for every $q \in \mathcal{P}_\mathrm{cond}$, with equality if and only if $q = p^*(\cdot \mid F_i, A)$ almost surely.
\item[(ii)] $\min_{q \in \mathcal{P}_\mathrm{blind}} \mathcal{R}(q) = H(s_i \mid F_i)$, attained at the antigen-marginalized conditional $p^*(a \mid F_i) = \mathbb{E}_{A \mid F_i}\bigl[p^*(a \mid F_i, A)\bigr]$.
\item[(iii)] The two optima differ by exactly $H(s_i \mid F_i) - H(s_i \mid F_i, A) = I(s_i; A \mid F_i) \ge 0$.
\end{enumerate}
The entire cross-entropy reduction that antigen conditioning can buy is therefore capped at $I(s_i; A \mid F_i)$, and the positional marginal is the cross-entropy optimum only inside the antigen-blind class. Whenever the antigen is informative, the marginal is strictly suboptimal, by precisely that amount.
\end{snugshade}

\begin{proof}
Write $p^*_i = p^*(\cdot \mid F_i, A)$. For any $q \in \mathcal{P}_\mathrm{cond}$, condition on $(F_i, A)$ and add and subtract $\log p^*_i(s_i)$ inside the expectation:
\begin{align*}
\mathcal{R}(q)
&= \mathbb{E}_{F_i, A}\, \mathbb{E}_{s_i \mid F_i, A}\bigl[-\log q(s_i \mid F_i, A)\bigr] \\
&= \mathbb{E}_{F_i, A}\, \mathbb{E}_{s_i \mid F_i, A}\!\left[-\log p^*_i(s_i) + \log \frac{p^*_i(s_i)}{q(s_i \mid F_i, A)}\right] \\
&= H(s_i \mid F_i, A) + \mathbb{E}_{F_i, A}\bigl[D_\mathrm{KL}(p^*_i \,\|\, q)\bigr].
\end{align*}
By Gibbs' inequality, the divergence term is non-negative and vanishes if and only if $q = p^*_i$ almost surely, which is the statement that cross-entropy is a strictly proper scoring rule. This gives (i).

For (ii), restrict to $q \in \mathcal{P}_\mathrm{blind}$. The same algebra applied with the conditioning variable $F_i$ alone yields $\mathcal{R}(q) = H(s_i \mid F_i) + \mathbb{E}_{F_i}[D_\mathrm{KL}(p^*(\cdot \mid F_i) \,\|\, q)]$, so the minimizer inside the class is $p^*(\cdot \mid F_i)$ and the minimum value is $H(s_i \mid F_i)$. The tower property gives $p^*(a \mid F_i) = \mathbb{E}_{A \mid F_i}[p^*(a \mid F_i, A)]$, so the antigen-blind optimum is the antigen-marginalized conditional and nothing else.

For (iii), subtract the two optima and apply the definition of conditional mutual information, $I(s_i; A \mid F_i) = H(s_i \mid F_i) - H(s_i \mid F_i, A)$, which is non-negative because conditioning on more variables cannot increase entropy. Note that (ii) identifies the marginal as optimal only within $\mathcal{P}_\mathrm{blind}$: as an element of $\mathcal{P}_\mathrm{cond}$ it incurs the excess risk $\mathbb{E}[D_\mathrm{KL}(p^*_i \,\|\, p^*(\cdot \mid F_i))] = I(s_i; A \mid F_i)$, which is strictly positive whenever $A$ is informative about $s_i$ given $F_i$.
\end{proof}

\begin{snugshade}\noindent
\textbf{Proposition \ref{prop:identifiability}} (Argmax Identifiability)
Let $s = (s_1, \dots, s_L)$ be the native CDR loop, and define the two greedy decoders
\begin{equation*}
\hat{s}_i = \argmax_{a \in \mathcal{A}} p^*(a \mid F_i, A),
\qquad
\bar{s}_i = \argmax_{a \in \mathcal{A}} p^*(a \mid F_i),
\end{equation*}
the first from the fully conditioned law and the second from its antigen marginal, with ties broken by a fixed rule. Let
\begin{equation*}
\mathcal{M} = \bigl\{ i \in \{1, \dots, L\} : \hat{s}_i \neq \bar{s}_i \bigr\}
\end{equation*}
be the positions at which the antigen moves the modal residue, and let $\mathrm{AAR}(u) = L^{-1} \sum_{i=1}^{L} \mathbb{1}[u_i = s_i]$. Then:
\begin{enumerate}
\item[(i)] $\bigl|\mathrm{AAR}(\hat{s}) - \mathrm{AAR}(\bar{s})\bigr| \le |\mathcal{M}| / L$, and the bound is attained.
\item[(ii)] The bound involves no information-theoretic quantity. In particular $\mathcal{M} = \emptyset$, hence $\mathrm{AAR}(\hat{s}) = \mathrm{AAR}(\bar{s})$ exactly, is compatible with $I(s_i; A \mid F_i)$ bounded away from zero.
\end{enumerate}
An antigen-blind and a perfectly conditioned decoder can therefore be indistinguishable under AAR while differing substantially in how much antigen information they carry, so AAR cannot certify conditioning and a gap in AAR cannot be read as a gap in conditioning.
\end{snugshade}

\begin{proof}
By the definition of $\mathcal{M}$ the two decodes agree for every $i \notin \mathcal{M}$, so the corresponding indicator terms cancel and
\begin{equation*}
\mathrm{AAR}(\hat{s}) - \mathrm{AAR}(\bar{s})
= \frac{1}{L}\sum_{i \in \mathcal{M}} \bigl(\mathbb{1}[\hat{s}_i = s_i] - \mathbb{1}[\bar{s}_i = s_i]\bigr).
\end{equation*}
Each summand is a difference of indicators and lies in $\{-1, 0, 1\}$, so the sum is at most $|\mathcal{M}|$ in absolute value, giving (i). The bound is attained when $\hat{s}_i = s_i \neq \bar{s}_i$ at every $i \in \mathcal{M}$.

For (ii), take $|\mathcal{A}| = 20$, let $A$ be uniform on $19$ values, and set $p^*(\cdot \mid F_i, A = k)$ to place mass $0.6$ on $a_1$ and $0.4$ on $a_{k+1}$. The mode is $a_1$ for every $k$, and the marginal $p^*(\cdot \mid F_i)$ places $0.6$ on $a_1$ and $0.4/19$ on each remaining residue, whose mode is also $a_1$. Hence $\mathcal{M} = \emptyset$ and the AAR gap is exactly zero, while $H(s_i \mid F_i, A) = 0.97$ bits and $H(s_i \mid F_i) = 2.67$ bits give $I(s_i; A \mid F_i) = 1.70$ bits, which is $39\%$ of the maximum $\log_2 20 = 4.32$ bits available at the position. The mechanism is that AAR is a $0$--$1$ loss and reads only the argmax, so every rearrangement of mass among the non-modal residues, which is where the antigen acts in this construction, is invisible to it.
\end{proof}

\begin{snugshade}\noindent
\textbf{Corrolary \ref{cor:vocab_collapse}} (Vocabulary Collapse under Argmax Decoding)
Let $c_i$ denote the full conditioning the model sees at position $i$ and let decoding be greedy, $\hat{s}_i = \argmax_{a} p_\theta(a \mid c_i)$. Evaluate over the index set $\mathcal{I}$ of all designed positions across all test complexes, and define the realized mode set, the empirical output distribution, and the effective vocabulary,
\begin{equation*}
\mathcal{S} = \{\hat{s}_i : i \in \mathcal{I}\},
\quad m = |\mathcal{S}|,
\quad \hat{q}(a) = \frac{1}{|\mathcal{I}|}\sum_{i \in \mathcal{I}} \mathbb{1}[\hat{s}_i = a],
\quad \mathrm{EV} = \exp\bigl(H(\hat{q})\bigr).
\end{equation*}
Then:
\begin{enumerate}
\item[(i)] $\mathrm{EV} \le m$, with equality if and only if $\hat{q}$ is uniform on $\mathcal{S}$.
\item[(ii)] The bound does not involve the per-position entropies $H\bigl(p_\theta(\cdot \mid c_i)\bigr)$. A model whose predictive distributions are individually close to uniform can still realize $m = 1$ and $\mathrm{EV} = 1$.
\end{enumerate}
Effective vocabulary therefore measures the diversity of the realized decoding rule, not the diversity of the learned distribution. A collapsed EV is evidence about the decoder, and it is not evidence that the model failed to learn a conditional distribution.
\end{snugshade}

\begin{proof}
Under greedy decoding $\hat{s}_i$ is a deterministic function of $c_i$, so the multiset of emitted residues is supported on $\mathcal{S}$ and $\hat{q}(a) = 0$ for $a \notin \mathcal{S}$. The entropy of a distribution supported on $m$ symbols satisfies $H(\hat{q}) \le \log m$, with equality if and only if $\hat{q}$ is uniform on the support, which is the maximum-entropy property of the uniform distribution on a finite set. Exponentiating both sides gives $\mathrm{EV} = \exp(H(\hat q)) \le m$, and the equality condition transfers because $\exp$ is strictly increasing. This is (i).

For (ii), fix any $c_i$-independent ordering of $\mathcal{A}$ and let $p_\theta(\cdot \mid c_i)$ place mass $1/20 + \epsilon$ on $a_1$ and distribute the rest evenly, for arbitrarily small $\epsilon > 0$. Every per-position distribution then has entropy within $O(\epsilon)$ of $\log 20$, yet the argmax is $a_1$ at every position, so $m = 1$ and $\mathrm{EV} = 1$. The bound in (i) depends on the realized modes alone, so no assumption about $p_\theta$ can improve it. In practice, for CDR-H3 the position-and-length null realizes $m = 8$ modes, and the baselines that top the AAR leaderboard fall below even that.
\end{proof}

\begin{snugshade}\noindent
\textbf{Proposition \ref{prop:equivariance}} (SE(3)-Equivariance of \af{})
Let the input be a complex with backbone coordinates $\mathbf{X} \in \mathbb{R}^{N \times 4 \times 3}$ (N, C$\alpha$, C, O per residue) and sequence $\mathbf{S}$. Let $g = (R, \mathbf{t})$ with $R \in \mathrm{SO}(3)$ and $\mathbf{t} \in \mathbb{R}^3$ act on coordinates atomwise by $g \cdot \mathbf{X} = R\mathbf{X} + \mathbf{t}$ and trivially on $\mathbf{S}$. Write the model as the pair of maps
\begin{equation*}
\bigl(\hat{\mathbf{Z}}, p\bigr) = \mathcal{F}(\mathbf{X}, \mathbf{S}),
\qquad
\hat{\mathbf{Z}} \in \mathbb{R}^{N \times 4 \times 3},
\quad
p \in \Delta(\mathcal{A})^{L},
\end{equation*}
with $\hat{\mathbf{Z}}$ the predicted coordinates and $p$ the per-position residue distributions. Then for every $g \in \mathrm{SE}(3)$ and every input:
\begin{enumerate}
\item[(i)] \emph{Structure is equivariant}: $\hat{\mathbf{Z}}(g \cdot \mathbf{X}, \mathbf{S}) = g \cdot \hat{\mathbf{Z}}(\mathbf{X}, \mathbf{S})$.
\item[(ii)] \emph{Sequence is invariant}: $p(g \cdot \mathbf{X}, \mathbf{S}) = p(\mathbf{X}, \mathbf{S})$.
\item[(iii)] \emph{The model is chiral}: (i) and (ii) fail for $g = (R, \mathbf{t})$ with $\det R = -1$, so the symmetry group of \af{} is $\mathrm{SE}(3)$ and not $\mathrm{E}(3)$.
\end{enumerate}
Claim (iii) is a design property and not a defect. L-amino acid backbones are chiral, so an $\mathrm{E}(3)$-invariant model would assign equal likelihood to a mirror-image complex that no protein can realize.
\end{snugshade}

\begin{proof}
We decompose the pipeline into three stages and verify each preserves the required symmetry.

\textbf{Stage 1: VN-EGNN.} Let $g = (R, \mathbf{t})$ with $R \in \mathrm{SO}(3)$ act on all input coordinates as $\mathbf{x}_i \mapsto R\mathbf{x}_i + \mathbf{t}$. Coordinate differences transform as $\Delta\mathbf{x}_{ij} \mapsto R \Delta\mathbf{x}_{ij}$ (translation cancels). The Gram matrix radial feature has entries $[\Delta\mathbf{x}_i \Delta\mathbf{x}_j^\top]_{ab} = \langle \Delta\mathbf{x}_i^{(a)}, \Delta\mathbf{x}_j^{(b)} \rangle$. Under rotation, $\langle R\mathbf{u}, R\mathbf{v} \rangle = \mathbf{u}^\top R^\top R \mathbf{v} = \langle \mathbf{u}, \mathbf{v} \rangle$, so the Gram matrix is invariant. Distances and RBFs of distances are likewise invariant. The quaternion orientation features are computed from residue local frames $[\mathbf{e}_1, \mathbf{e}_2, \mathbf{e}_1 \times \mathbf{e}_2]$; each frame transforms as $F \mapsto RF$, so the relative rotation $F_i^\top F_j$ and the quaternion read off from it are invariant for $R \in \mathrm{SO}(3)$. The message $\mathbf{m}_{ij} = \text{MLP}([\mathbf{h}_i, \mathbf{h}_j, \text{invariant features}])$ is therefore invariant, and the node update $\mathbf{h}_i'$ is invariant. The coordinate update $\mathbf{x}_i' = \mathbf{x}_i + \sum_j \Delta\mathbf{x}_{ij} \odot f(\mathbf{m}_{ij})$ is equivariant: under $g$, $R\mathbf{x}_i + \mathbf{t} + \sum_j R\Delta\mathbf{x}_{ij} \odot f(\mathbf{m}_{ij}) = R\mathbf{x}_i' + \mathbf{t}$.

This establishes (iii) as well. For $R \in \mathrm{O}(3)$ with $\det(R) = -1$ the Gram matrix and the distances remain invariant, but the third frame axis obeys $(R\mathbf{e}_1) \times (R\mathbf{e}_2) = \det(R)\, R(\mathbf{e}_1 \times \mathbf{e}_2) = -R(\mathbf{e}_1 \times \mathbf{e}_2)$, so the reflected frame is $R \, \mathrm{diag}(1,1,-1) F$ rather than $RF$. The relative rotation $F_i^\top F_j$ therefore changes, the vector part of the quaternion read off from it flips sign, and the messages $\mathbf{m}_{ij}$ are not preserved. Both the coordinate stream and the sequence head consume those messages, so neither (i) nor (ii) survives a reflection and the symmetry group is exactly $\mathrm{SE}(3)$.

\textbf{Stage 2: Hyperbolic cross-attention.} The module receives GNN node features $\mathbf{h}$ (invariant from Stage 1). All operations (linear projections, hyperboloid mapping, Minkowski inner product, softmax, value aggregation) act on these invariant features without accessing coordinates. The output is therefore invariant.

\textbf{Stage 3: MDN-Potts head.} The sequence head receives the concatenation of invariant features from Stages 1--2 that are functions of the sequence alone and hence trivially invariant. All operations (layer norm, linear projections, softmax, belief propagation) preserve invariance. The framework dropout masks feature channels and does not touch coordinates, so it commutes with the group action. The predicted sequence probabilities $p(\mathbf{s} \mid \mathbf{X}, \mathbf{S})$ are therefore SE(3)-invariant, establishing (ii). The coordinate predictions $\hat{\mathbf{Z}}$ come from the GNN's equivariant coordinate stream, establishing (i).
\end{proof}



\subsection{Graph Construction Details}
\label{app:graph}

\begin{table}[h!]
\centering
\caption{Edge types in the heterogeneous antibody-antigen graph.}
\label{tab:edge_types}
\small
\setlength{\tabcolsep}{4pt}
\begin{tabular}{clll}
\toprule
Type & Edge class & Scope & Criterion \\
\midrule
0 & Intra-chain radial & Same chain & C$\alpha$ dist $< 8$~\AA \\
1 & Global$\leftrightarrow$normal & Global tokens$\leftrightarrow$chain & Segment membership \\
2 & Global$\leftrightarrow$global & Global tokens & All pairs \\
3 & Sequential $d{=}1$ & Same chain (Ab only) & Adjacent residues \\
4 & Intra-chain KNN & Same chain & $K{=}8$ nearest C$\alpha$ \\
5 & Sequential $d{=}2$ & Same chain (Ab only) & Distance-2 residues \\
6 & Inter-chain radial & Different chains & C$\alpha$ dist $< 12$~\AA \\
7 & Inter-chain KNN & Different chains & $K{=}8$ nearest C$\alpha$ \\
8 & VN$\leftrightarrow$epitope & Virtual nodes$\leftrightarrow$epitope & Bidirectional, all pairs \\
9 & VN$\leftrightarrow$CDR & Virtual nodes$\leftrightarrow$CDR & Bidirectional, all pairs \\
\bottomrule
\end{tabular}
\end{table}

Each residue is represented by a 108-dimensional feature vector: sequential position embeddings (16D via sinusoidal encoding), distance radial basis functions for three backbone bond lengths N--C$\alpha$, C$\alpha$--C, and C--O (each expanded into 16 Gaussian basis functions, totaling 48D), backbone dihedral and bond angles as sine-cosine pairs (12D), local coordinate frame directional features constructed from backbone unit vectors rotated into a local reference frame (9D), amino acid identity (20D, masked to zero for CDR positions during training), and interface complementarity features (4D). These 105 features are concatenated with a 3D segment type embedding that distinguishes heavy chain, light chain, and antigen residues. A dual-path encoder processes the dense features (101D) and sparse complementarity features (4D) through separate two-layer MLPs with SiLU activations, fuses them via concatenation, and projects the result to the embedding dimension $d = 128$. The epitope residues receive an additional learnable embedding (zero-initialized) added to the projected features.

Each edge carries a 104-dimensional feature vector composed of a type-specific one-hot encoding (8D for the 8 standard types), relative position embeddings (16D), pairwise distance radial basis functions for four atom pairs N--C$\alpha$, C$\alpha$--C$\alpha$, C--C$\alpha$, and O--C$\alpha$ (each expanded into 16 Gaussian basis functions, totaling 64D), quaternion-encoded relative orientation (4D), and local frame directional features for the four atom pairs (12D). The positional features are set to zero for inter-chain and global edges (types 1, 2, 6, 7) since relative sequence position is undefined across chains. The virtual node edges (types 8--9) use learnable feature vectors initialized from a small random distribution. Moreover, all the interface metrics use symmetric C$\alpha$--C$\alpha$ contacts at 8~\AA{} restricted to CDR residues.

\subsection{Algorithms}
\label{app:algorithms}

The training algorithm for \af{} is presented in Algorithm~\ref{alg:training}, while the inference algorithm is in Algorithm~\ref{alg:inference}.

\begin{algorithm}[h!]
\caption{Training}
\label{alg:training}
\begin{algorithmic}[1]
\REQUIRE Dataset $\mathcal{D}$, model $f_\theta$, epochs $T$, aMCL schedule $\tau(t)$
\FOR{$t = 1$ to $T$}
    \STATE $\tau \gets \tau_\text{start} \cdot (\tau_\text{end}/\tau_\text{start})^{t/T_\text{anneal}}$
    \FOR{each batch $\mathcal{B} \subset \mathcal{D}$}
        \STATE Apply framework dropout: mask heavy-chain embeddings with probability $p$
        \STATE Build heterogeneous graph $\mathcal{G}$ with 10 edge types
        \STATE $\{\mathbf{h}_i, \hat{\mathbf{x}}_i\} \gets \text{VN-EGNN}(\mathcal{G})$ \hfill \COMMENT{5 layers}
        \STATE $\mathbf{z}_i \gets \text{HypAttn}(\mathbf{h}_\text{CDR}, \mathbf{h}_\text{epi})$ \hfill \COMMENT{Lorentz hyperboloid}
        \STATE $\mathbf{f}_i \gets [\mathbf{h}_i; \mathbf{z}_i; \text{ESM}_\text{proj}(\mathbf{e}_i)]$ \hfill \COMMENT{768-dim}
        \STATE $\{\boldsymbol{\ell}_k, \boldsymbol{\pi}\} \gets \text{MDN-Potts}(\mathbf{f})$ \hfill \COMMENT{$K$=4, 2-round BP}
        \STATE Compute $w_k = \text{softmax}(-\ell_k/\tau)$ \hfill \COMMENT{aMCL weights}
        \STATE $\mathcal{L} \gets \mathcal{L}_\text{seq} + \alpha\mathcal{L}_\text{coord} + \delta\mathcal{L}_\text{shadow} + \epsilon\mathcal{L}_\text{GDPP} + \lambda_\text{cls}\mathcal{L}_\text{cls}$
        \STATE Update $\theta$ via AdamW with gradient clipping (max norm 0.5)
    \ENDFOR
    \STATE Decay learning rate: $\eta \gets \eta \cdot \gamma$
    \STATE Evaluate on validation set; early stop if no improvement for 10 epochs
\ENDFOR
\end{algorithmic}
\end{algorithm}

\begin{algorithm}[h!]
\caption{Inference}
\label{alg:inference}
\begin{algorithmic}[1]
\REQUIRE Trained model $f_\theta$, test complex with antigen + antibody framework
\STATE Build graph $\mathcal{G}$ (CDR AA masked, coordinates from framework)
\STATE $\{\mathbf{h}_i, \hat{\mathbf{x}}_i\} \gets \text{VN-EGNN}(\mathcal{G})$
\STATE $\mathbf{z}_i \gets \text{HypAttn}(\mathbf{h}_\text{CDR}, \mathbf{h}_\text{epi})$
\STATE $\mathbf{f}_i \gets [\mathbf{h}_i; \mathbf{z}_i; \text{ESM}_\text{proj}(\mathbf{e}_i)]$
\STATE $\{\boldsymbol{\ell}_k, \boldsymbol{\pi}\} \gets \text{MDN-Potts}(\mathbf{f})$
\FOR{each CDR position $i$}
    \STATE $k^* \gets \arg\max_k \pi_k(i)$ \hfill \COMMENT{Best component}
    \STATE $\hat{s}_i \gets \arg\max_a \ell_{k^*}^{(i)}(a)$ \hfill \COMMENT{Greedy decode}
\ENDFOR
\RETURN Predicted sequence $\hat{\mathbf{s}}$, predicted coordinates $\hat{\mathbf{X}}_\text{CDR}$
\end{algorithmic}
\end{algorithm}

\subsection{Implementation Details}
\label{app:implementation}

\af{} is a 5-layer VN-EGNN with embedding dimension 128, hidden dimension 256, and 3 virtual nodes, connecting nodes through 10 edge types (8 standard plus 2 for virtual-node connectivity). The cross-attention on the Lorentz hyperboloid has 4 heads at curvature $c = 1.0$, with the epitope GNN hidden state providing the keys and values. The MDN-Potts head has $K = 4$ mixture components (swept) refined by 2 rounds of belief propagation over adjacent-position coupling matrices, a shared hidden dimension of 384, and a pairwise-coupling weight $\lambda_\text{pair} = 0.3$. Annealed multiple choice learning (aMCL) anneals the temperature exponentially from $\tau_\text{start} = 2.0$ to $\tau_\text{end} = 0.1$ over 20 epochs, with the endpoints and the anneal window all swept, and the exponential schedule is fixed. For training, we used AdamW with learning rate $2.2 \times 10^{-4}$, per-epoch decay $\gamma = 0.955$, batch size 8, a maximum of 50 epochs, and early stopping with patience 10 on the validation loss. The loss weights are $\alpha = 1.301$ on the coordinate term, $\delta = 0.664$ on the shadow-paratope term, $\epsilon = 0.05$ on the GDPP diversity term, and $\lambda_\text{cls} = 0.2$ on the antigen classification term, while the framework dropout is applied at rate $p = 0.3$. The hyperparameters were tuned by sweeping experiments in Weights and Biases (W\&B). It takes approximately two hours for a full training run on a single H100.

\subsection{Additional Results}

\subsubsection{Epitope-group and temporal splits}
\label{app:splits}

Table~\ref{tab:other_splits} repeats the CDR-H3 comparison on the two remaining \chimera{} splits. The epitope-group split holds out clusters of related epitopes and the temporal split holds out complexes deposited after a cutoff date. \af{} leads fnat, iRMSD, DockQ, and epitope F1 on both splits, while adding to the perplexity and RMSD. MEAN again leads AAR on epitope-group (0.42 versus 0.39) and ties \af{} on temporal (0.39), which reproduces the recovery-versus-collapse trade-off of Table~\ref{tab:main_results}. \af{} therefore keeps the lead on the interface and structural metrics on all \chimera{} splits, while maintaining robust AAR and effective vocabulary. 

\begin{table}[h!]
\centering
\caption{CDR-H3 design performance comparison of \af{} with the baselines on the epitope-group and temporal splits of \chimera{}. Best values are in \textbf{bold}, second-best \underline{underlined}. }
\label{tab:other_splits}
\small
\setlength{\tabcolsep}{2.5pt}
\begin{tabular}{llccccccc}
\toprule
& Method & AAR$\uparrow$ & PPL$\downarrow$ & RMSD$\downarrow$ & fnat$\uparrow$ & iRMSD$\downarrow$ & DockQ$\uparrow$ & epiF1$\uparrow$ \\
\midrule
\multirow{11}{*}{\rotatebox{90}{\small Epitope-group}}
& RAAD & 0.38 & 3.31 & \underline{1.95} & 0.49 & \underline{1.64} & \underline{0.64} & \underline{0.67} \\
& MEAN & \textbf{0.42} & \underline{3.00} & 2.01 & 0.48 & 1.66 & 0.63 & \underline{0.67} \\
& dyMEAN & 0.35 & -- & 2.60 & 0.41 & 2.23 & 0.56 & 0.55 \\
& DiffAb & 0.22 & -- & 2.64 & 0.48 & 2.27 & 0.58 & 0.56 \\
& AbFlowNet & 0.22 & -- & 2.70 & 0.49 & 2.37 & 0.57 & 0.55 \\
& AbMEGD & 0.22 & -- & 2.76 & \underline{0.50} & 2.41 & 0.57 & 0.56 \\
& RADAb & 0.22 & -- & 12.28 & 0.47 & 5.82 & 0.57 & 0.57 \\
& dyAb & 0.27 & -- & 3.31 & 0.35 & 2.27 & 0.54 & 0.50 \\
& AbODE & 0.31 & 9.99 & 16.40 & 0.08 & 5.12 & 0.34 & 0.20 \\
& AbDockGen & 0.23 & 7.89 & 3.93 & 0.34 & 2.64 & 0.52 & 0.62 \\
& \textbf{\af{} (\textit{ours})} & \underline{0.39} & \textbf{2.99} & \textbf{1.88} & \textbf{0.53} & \textbf{1.56} & \textbf{0.66} & \textbf{0.70} \\
\midrule
\multirow{11}{*}{\rotatebox{90}{\small Temporal}}
& RAAD & \underline{0.38} & \underline{3.24} & \underline{1.81} & \underline{0.61} & \underline{1.46} & \underline{0.71} & \underline{0.74} \\
& MEAN & \textbf{0.39} & 3.31 & 1.84 & 0.60 & 1.44 & \underline{0.71} & 0.73 \\
& dyMEAN & 0.35 & 3.28 & 2.32 & 0.50 & 2.00 & 0.63 & 0.64 \\
& DiffAb & 0.23 & -- & 2.46 & 0.51 & 2.18 & 0.61 & 0.61 \\
& AbFlowNet & 0.23 & -- & 2.67 & 0.52 & 2.33 & 0.61 & 0.61 \\
& AbMEGD & 0.23 & -- & 2.72 & 0.55 & 2.35 & 0.62 & 0.58 \\
& RADAb & 0.22 & -- & 11.57 & 0.48 & 7.91 & 0.57 & 0.57 \\
& dyAb & 0.29 & -- & 2.89 & 0.51 & 1.92 & 0.64 & 0.65 \\
& AbODE & 0.37 & 9.26 & 15.63 & 0.07 & 4.65 & 0.35 & 0.19 \\
& AbDockGen & 0.23 & 7.55 & 3.89 & 0.37 & 2.52 & 0.55 & 0.68 \\
& \textbf{\af{} (\textit{ours})} & \textbf{0.39} & \textbf{3.07} & \textbf{1.71} & \textbf{0.64} & \textbf{1.36} & \textbf{0.73} & \textbf{0.76} \\
\bottomrule
\end{tabular}
\end{table}

\subsubsection{Per-CDR results}
\label{app:per_cdr}

\begin{table}[h!]
\centering
\caption{Per-CDR comparison on \chimera{} (antigen-fold split, means only). AAR and C$\alpha$ RMSD (\AA) for each heavy-chain CDR. Best in \textbf{bold}, second-best \underline{underlined}.}
\label{tab:per_cdr}
\small
\setlength{\tabcolsep}{4pt}
\begin{tabular}{lcccccc}
\toprule
& \multicolumn{2}{c}{CDR-H1} & \multicolumn{2}{c}{CDR-H2} & \multicolumn{2}{c}{CDR-H3} \\
\cmidrule(lr){2-3} \cmidrule(lr){4-5} \cmidrule(lr){6-7}
Method & AAR$\uparrow$ & RMSD$\downarrow$ & AAR$\uparrow$ & RMSD$\downarrow$ & AAR$\uparrow$ & RMSD$\downarrow$ \\
\midrule
\textbf{\af{}} & \textbf{0.70} & \textbf{0.50} & \underline{0.59} & \textbf{0.45} & \underline{0.39} & \textbf{1.55} \\
RAAD & \underline{0.69} & \underline{0.54} & \textbf{0.60} & \underline{0.49} & 0.38 & \underline{1.65} \\
MEAN & 0.68 & 0.77 & 0.54 & 1.01 & \textbf{0.40} & 1.66 \\
dyMEAN & 0.67 & 0.99 & 0.53 & 1.37 & 0.37 & 2.06 \\
DiffAb & 0.59 & 0.84 & 0.26 & 0.90 & 0.23 & 2.24 \\
AbFlowNet & 0.57 & 0.80 & 0.28 & 0.82 & 0.23 & 2.70 \\
AbMEGD & 0.61 & 0.82 & 0.28 & 0.71 & 0.24 & 2.25 \\
RADAb & 0.65 & 3.18 & 0.36 & 0.86 & 0.25 & 7.79 \\
dyAb & 0.52 & 1.70 & 0.41 & 1.72 & 0.28 & 2.40 \\
AbDockGen & -- & -- & -- & -- & 0.24 & 3.65 \\
AbODE & 0.45 & 4.82 & 0.47 & 7.27 & 0.27 & 13.17 \\
\bottomrule
\end{tabular}
\end{table}

\af{} matches or leads the baselines on the shorter, more conserved CDR-H1 and CDR-H2 loops. On CDR-H1, it leads AAR at 0.70 against RAAD's 0.69 with the best RMSD (0.50~\AA{} versus 0.54~\AA). On CDR-H2, it reaches AAR 0.59 against RAAD's 0.60 with the best RMSD at 0.45~\AA. These loops are structurally more constrained than H3, so all competent methods cluster. The gap between methods widens on H3, where the design problem remains an open research area.

\subsubsection{Antigen conditioning analysis}
\label{app:conditioning}




We test the observed failure modes by fitting a null predictor $p(a \mid \text{position bin}, \text{loop length})$ on each training split, which sees neither the antigen nor the structure. On the antigen-fold split, the null reaches perplexity 9.81 and AAR 0.349 while emitting only 8 distinct amino acids in total; the epitope-group and temporal splits give 9.83/0.339 and 9.75/0.337. Every trained method drives perplexity far below the null, which is a direct evidence that conditioning is learned. 
None of our tested methods improves AAR over the null by more than 0.05, where methods scoring higher AAR emit \emph{fewer} distinct residues than the null itself (effective vocabulary 3.7 to 6.9 against a native 14.9). 

The native H3 residue distribution has an effective vocabulary of 14.9. Every method that leads the AAR table collapses it: MEAN to 6.7, \af{} to 8.9, RAAD to 6.0, dyMEAN to 4.6, dyAb to 6.9, AbODE to 3.7. The methods that preserve native diversity (AbDockGen 15.3, DiffAb and AbFlowNet 12.1, AbMEGD 11.8, RADAb 11.6) have significantly lower AAR. The six lowest-EV methods are the six highest-AAR methods. Recovery and diversity are inversely related on this task, as Corollary~\ref{cor:vocab_collapse} predicts, because argmax decoding emits only the modes a model realizes. \af{} preserves 33\% more vocabulary than MEAN and gives up 2 points of AAR for it, which places it on the better side of this trade-off. The bigram statistics confirm the loss of higher-order structure, where native H3 has bigram entropy 5.14 over 353 unique bigrams, \af{} reaches 3.71 over 176, and MEAN 3.02 over 93. Figure~\ref{fig:vocab-collapse} visualizes the collapse across methods on the epitope group split.


\begin{figure}[h!]
    \centering
    \includegraphics[width=1\linewidth]{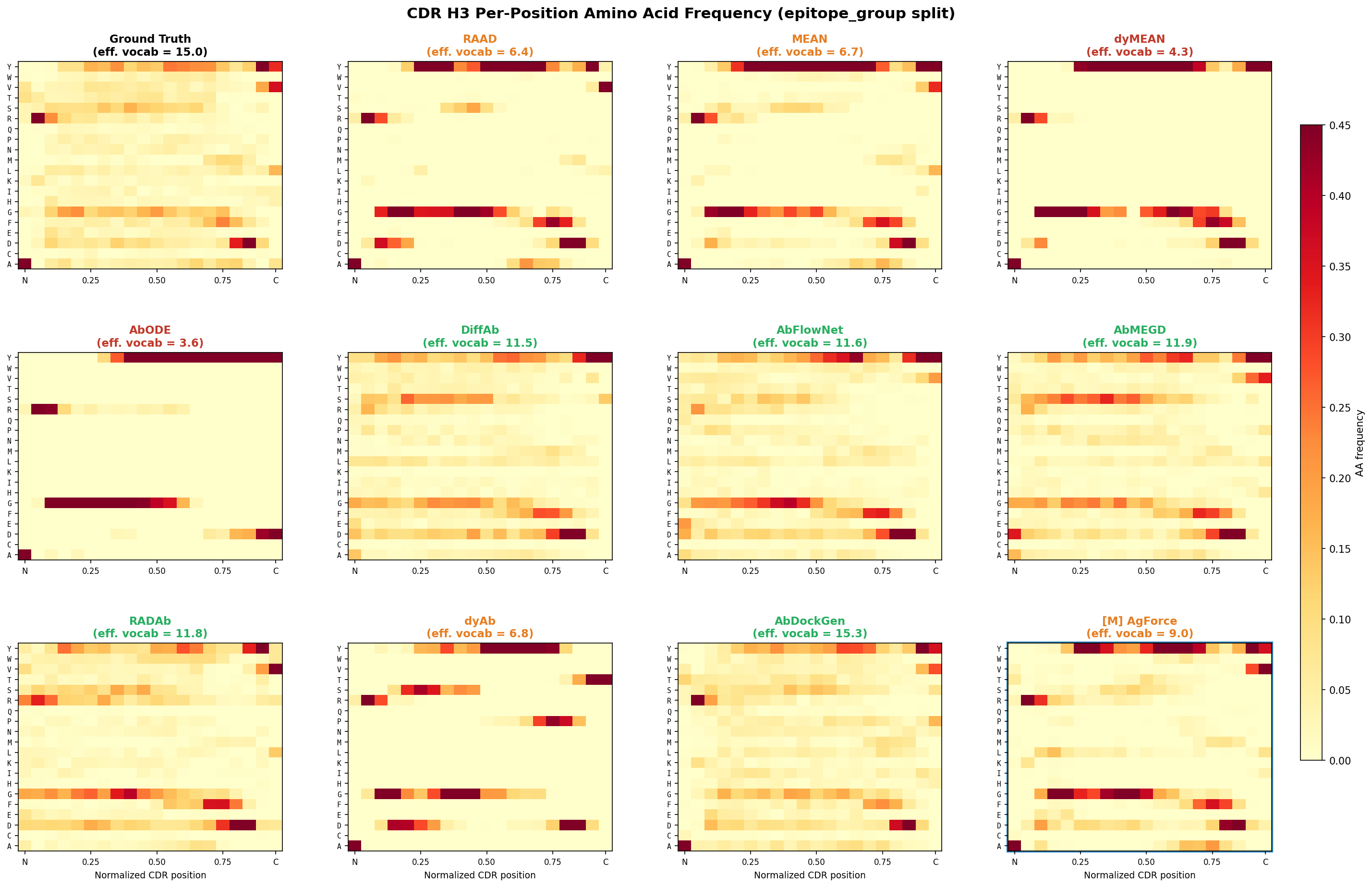}
    \caption{Illustration of the vocabulary collapse failure mode of the baselines vs. \af{}.}
    \label{fig:vocab-collapse}
\end{figure}

\end{document}